\documentclass[final]{siamonline171218}

\usepackage{url}            
\usepackage{hyperref}

\usepackage{algorithm, algorithmic}
\renewcommand{\algorithmicrequire}{\textbf{Input:}} 
\renewcommand{\algorithmicensure}{\textbf{Output:}}

\usepackage{amssymb}
\usepackage{pifont}

\usepackage{float,graphicx,epstopdf}
\usepackage{bbm}
\usepackage{algorithm, algorithmic}
\renewcommand{\algorithmicrequire}{\textbf{Input:}} 
\renewcommand{\algorithmicensure}{\textbf{Output:}}

\usepackage{dsfont,amssymb,amsmath,graphicx,enumitem, subfig, multicol}
\usepackage{amsfonts,dsfont,mathtools, mathrsfs,amsthm,fancyhdr} 
\usepackage{enumitem}
\allowdisplaybreaks
\usepackage{pgfplotstable}
\usepackage{booktabs}
\usepackage{wrapfig}
\usepackage{stmaryrd}
\makeatletter
\def\munderbar#1{\underline{\sbox\tw@{$#1$}\dp\tw@\z@\box\tw@}}
\makeatother

\newtheorem{theorem}{Theorem}[section]

\newtheorem{remark}[theorem]{Remark}

\newtheorem{proposition}[theorem]{Proposition}

\newtheorem{assumption}[theorem]{Assumption}

\newcommand{\be}{\begin{equation}}
\newcommand{\ee}{\end{equation}}
\newcommand{\bea}{\begin{equation*}\begin{aligned}}
\newcommand{\eea}{\end{aligned}\end{equation*}}

\newcommand{\R}{\mathbb{R}}

\newcommand{\Max}{\max\limits_}
\newcommand{\Min}{\min\limits_}

\newcommand{\mc}{\mathcal}
\newcommand{\mbb}{\mathbb}
\newcommand{\inner}[2]{\big \langle #1, #2 \big \rangle }

\DeclareMathOperator{\st}{s.t.}

\newcommand{\Sym}{\mathbb{S}}
\newcommand{\PSD}{\mathbb{S}_{+}} 
\newcommand{\PD}{\mathbb{S}_{++}} 
\newcommand{\Let}{\triangleq}
\newcommand{\opt}{^\star}
\newcommand{\eps}{\varepsilon}

\newcommand{\half}{\frac{1}{2}}

\newcommand{\cmark}{\ding{51}}
\newcommand{\xmark}{\ding{55}}

\headers{Cost-Adaptive Recourse Recommendation by Adaptive Preference Elicitation}{D. Nguyen, B. Nguyen, and V. A. Nguyen}

\title{Cost-Adaptive Recourse Recommendation by Adaptive Preference Elicitation}

\author{Duy Nguyen\thanks{VinAI Research, Vietnam 
  (\email{v.duynk13@vinai.io})}
\and Bao Nguyen\thanks{VinUniversity
  (\email{bao.nn2@vinuni.edu.vn})}
\and Viet Anh Nguyen\thanks{The Chinese University of Hong Kong
  (\email{nguyen@se.cuhk.edu.hk})}}

\usepackage{natbib}

\pgfplotsset{compat=1.18}
\begin{document}
\maketitle
\begin{abstract}
Algorithmic recourse recommends a cost-efficient action to a subject to reverse an unfavorable machine learning classification decision. Most existing methods in the literature generate recourse under the assumption of complete knowledge about the cost function. In real-world practice, subjects could have distinct preferences, leading to incomplete information about the underlying cost function of the subject. This paper proposes a two-step approach integrating preference learning into the recourse generation problem. In the first step, we design a question-answering framework to refine the confidence set of the Mahalanobis matrix cost of the subject sequentially. Then, we generate recourse by utilizing two methods: gradient-based and graph-based cost-adaptive recourse that ensures validity while considering the whole confidence set of the cost matrix. The numerical evaluation demonstrates the benefits of our approach over state-of-the-art baselines in delivering cost-efficient recourse recommendations.
\end{abstract}


\section{Introduction} \label{sec:intro}
Many machine learning algorithms are deployed to aid essential decisions in various domains. These decisions might have a direct or indirect influence on people's lives, especially in the case of high-profile applications~\citep{ref:verma2022counterfactual} such as job hiring~\citep{ref:harris2018making, ref:pessach2020employees}, bank loan~\citep{ref:wang2020comparative, ref:turkson2016machine} and medical diagnosis~\citep{ref:fatima2017survey, ref:latif2019medical}. Thus, it is imperative to develop methods to explain the prediction of machine learning models. For instance, when a person applies for a job and is rejected by a predictive model deployed by the employer, the applicant should be notified of the reasoning behind the unfavorable decision and what they could do to be hired in future applications. In a medical context, a machine learning model is utilized to predict whether or not a person will suffer from a stroke in the future using their current medical record and habits. If a person receives an undesirable health prediction from the model (e.g., high risk for stroke), it is essential to provide rationality for that diagnosis and additional medical guidance or lifestyle alternation to mitigate this condition. 

Recently, algorithmic recourse has become a powerful tool for explaining machine learning (ML) models. Recourse refers to the actions a person should take to achieve an alternative predicted outcome, and it is also known in the literature as a counterfactual, or prefactual, explanation. In the case of job hiring, recourse should be individualized suggestions such as ``get two more engineering certificates'' or ``complete one more personal project.'' Regarding the stroke prediction model, the recourse should be medical advice such as ``keep the ratio of sodium in the blood below 145 mEq/L'' or ``increase the water consumption per day up to 2 liters.''  When a company suggests a recourse to a job applicant, this recourse must be valid because the company should accept all applicants who completely implement the suggestions provided in the recommended recourse. A similar requirement holds for the health improvement recourse. Throughout this paper, we use ``subject'' to refer to the individual subject to the algorithm's prediction. In our job-hiring example, ``subject'' refers to the job applicant the company rejected.

Several approaches have been proposed to generate recourse for a machine learning model prediction~\citep{ref:karimi2022survey, ref:verma2022counterfactual, ref:steplin2021survey}.~\citet{ref:wachter2018counterfactual} used gradient information of the underlying model to generate a counterfactual closest to the input.~\citet{ref:ustun2019actionable} introduced an integer programming problem to find the minimal and actionable change for an input instance. \citet{ref:pawelczyk2020learning} leveraged the ideas from manifold learning literature to generate counterfactuals on the high-density data region. \citet{ref:karimi2020algorithmic, ref:karimi2021algorithmic} generated counterfactual as a sequence of interventions based on a pre-defined causal graph.

We provide an example of a gradient-based and a graph-based recourse to the stroke prediction example in Figure~\ref{fig:medical}. In the simplest form, a recourse dictates only the terminal state under which the algorithmic model will output a desirable prediction, see Figure~\ref{fig:medical}-left. A sequential recourse consists of multiple smaller actions that guide the subject toward a desirable terminal state, see Figure~\ref{fig:medical}-right. Observe that the terminal states differ due to the algorithms used to find the recourse in each setting. 
\begin{figure*}[!ht]
        \centering
        \includegraphics[width=0.9\linewidth]{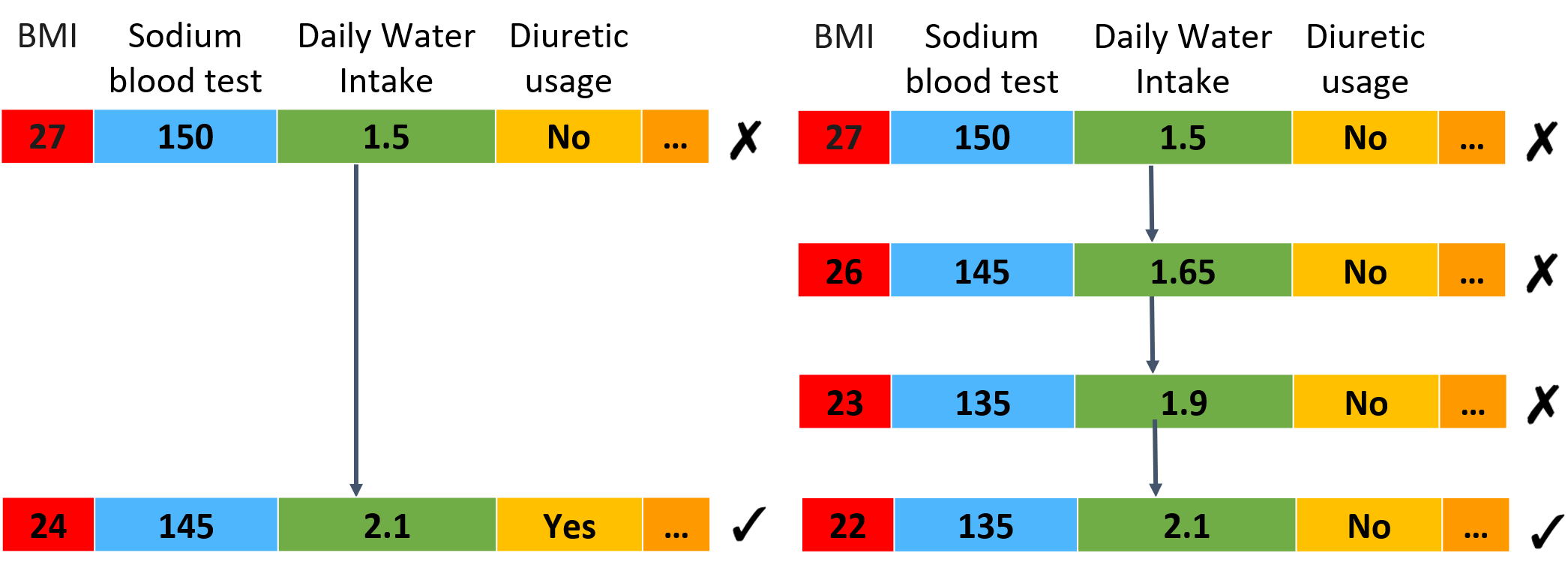}
        \caption{Example of a gradient-based one-step recourse recommendation (left) and a graph-based sequential recourse recommendation (right) on the stroke prediction example. \xmark \space denotes the unfavorable outcomes, and \cmark \space denotes the favorable outcomes.}
        \label{fig:medical}
\end{figure*}

These aforementioned approaches all assume that all subjects have the same cost function, for example, the $l_1$ distance~\citep{ref:ustun2019actionable, ref:upadhyay2021towards, ref:slack2021counterfactual, ref:ross2021learning} or define the same prior causal graph for all subjects~\citep{ref:karimi2020algorithmic, ref:karimi2021algorithmic}. This assumption results in two subjects with identical attributes receiving the same recourse recommendation. Unfortunately, this recourse recommendation is unrealistic in practice because having identical attributes does not necessarily guarantee that the two subjects will have identical preferences. Indeed, human preferences are strongly affected by many unobservable factors, including historical and societal experiences, which are hardly encoded in the attributes. Thus, the cost functions could be significantly different even between subjects with identical attributes, yet this difference is rarely considered in the recourse generation framework~\citep{ref:yadav2021low}. 

To mitigate these issues,~\citet{ref:toni2022generating} proposed a human-in-the-loop framework to generate a counterfactual explanation uniquely suited to each subject. The proposed method first fixes the initialized causal graph and iteratively learns the subject’s specific cost function. Recourse is generated by a reinforcement learning approach that searches for a suitable sequence of interventions. The disadvantage of this approach is that it requires a pre-defined causal graph, which is rarely available in practice~\citep{ref:verma2022counterfactual}. Besides,~\citet{ref:rawal2020beyond} employed the Bradley-Terry model to estimate a universal cost function and then utilized the user input to generate personalized recourse for the user. However, this method is additive in features; therefore, its ability to recover the underlying causal graph remains problematic. Following the same line of work, ~\citet{ref:yetukuri2023actionable} captures user preferences via three soft constraints: scoring continuous features, bounding feature values, and ranking categorical features. This method generates recourse via a gradient-based approach. However, the fractional-score concept for user preference might not be as straightforward, especially when the data has many continuous features.

To resolve these problems, we propose a preference elicitation framework that learns the subject's cost function from pairwise comparisons of possible recourses. Compared to~\citet{ref:toni2022generating}, our framework does not require the causal graph as input, and compared to~\citet{ref:rawal2020beyond} and~\citet{ref:yetukuri2023actionable}, our framework can perform well even when the dimension of the feature space grows large. This paper contributes by: 

\begin{itemize}[leftmargin=5mm]
\item proposing in Section~\ref{sec:cost-identify} an adaptive preference learning framework to learn the subject's cost function parametrized by the cost matrix of a Mahalanobis distance. This framework initializes with an uninformative confidence set of possible cost matrices. In each round, it determines the next question by finding a pair of recourses corresponding to the most effective cut of the confidence set, that is, a cut that slices the incumbent confidence set most aggressively. The incumbent confidence set shrinks along iterations. We terminate the questioning upon reaching a predefined number of inquiries. The final confidence set is employed for recourse generation.
\item proposing in Section~\ref{sec:generation} two methods for generating recourse under various assumptions of the machine learning models. These methods will consider explicitly the terminal confidence set about the subject's cost matrix. If the model is white-box and differentiable, we can use the cost-adaptive gradient-based recourse-generation method that generates cost-adaptive recourse. Otherwise, we can use the graph-based method to generate the sequential recourse.
\end{itemize}

In Section~\ref{sec:generalization}, we extend our framework to cope with potential inconsistencies in subject responses and extend the heuristics from pairwise comparison to multiple-option questions. Section~\ref{sec:experiment} reports our numerical results, which empirically demonstrate the benefit of our proposed approach. 

\textbf{Notations.} Given an integer $d$, we use $\mbb S^d$ and $\PSD^d$ to denote the space of $d$-by-$d$ symmetric matrices and $d$-by-$d$ symmetric positive definite matrices, respectively. The identity matrix is denoted by $I$. The inner product between two matrices $A, B \in \mbb S^d$ is $\inner{A}{B} =  \sum_{i,j} A_{ij} B_{ij}$, and we write $A \preceq B$ to denote that $B - A \in \PSD^d$. The set of integers from 1 to $N$ is $\llbracket N \rrbracket$. 

\section{Problem Statement and Solution Overview} \label{sec:overview}

We are given a binary classifier $\mc C_{\theta}: \R^d \to \{0, 1\}$ and access to the training dataset containing $N + M$ instances $x_i \in \R^d$, $i = 1, \ldots, N + M$. The dataset is split into two parts: 
\begin{itemize}[leftmargin=5mm]
    \item a positive dataset $\mc D_1=\{x_1,\ldots,x_N\}$ containing instances with $\mc C_{\theta}(x_i)=1~\forall x_i \in \mc D_1$.
    \item a negative dataset $\mc D_0=\{x_{N+1},\ldots,x_{N+M}\}$ containing all instances that have the negative predicted outcome, thus $\mc C_{\theta}(x_i)=0~\forall x_i \in \mc D_0$. 
\end{itemize}
Given a subject with input $x_0 \in \R^d$ with a negative predicted outcome $\mc C_{\theta}(x_0) = 0$, we make the following assumption on the cost function of this subject.
\begin{assumption}
    The subject $x_0$ has a Mahalanobis cost function of the form $c_{A_0}(x, x_0) = (x - x_0)^\top A_0 (x - x_0)$ for some symmetric, positive definite matrix $A_0 \in \PD^d$. 
\end{assumption}

We provide two possible justifications for the aforementioned assumption in Appendix~\ref{sec:motivation-M}. First, we describe a sequential control process that affects feature transitions of a subject $x_0$ towards a recourse $x_r$ while minimizing the cost of efforts. We formalize this problem as a Linear Quadratic Regulator, and then we prove that the optimal cost function has the Mahalanobis form, see Section~\ref{sec:LQR}. Second, Appendix~\ref{sec:causal} establishes a connection between the linear Gaussian structural causal model and the Mahalanobis cost function. We show that we can recover the Mahalanobis cost preference model with $A_0$ corresponding to the precision matrix of the deviation under linear Gaussian structural equation assumption.

In the above cost function, $A_0$ is the ground-truth matrix specific for subject $x_0$, but it remains elusive to the recourse generation framework. We aim to find $x_r$ which has a positive predicted outcome $\mc C_{\theta}(x_r) = 1$ and minimizes the cost $c_{A_0}(x_r, x_0)$. Because the matrix $A_0$ is unknown, we propose an adaptive preference learning approach~\citep{ref:bertsimas2013learning, ref:vayanos2020robust} to approximate the actual cost function $c_{A_0}(x, x_0)$. Our overall approach is as follows: We have a total of $T$ question-answer rounds for cost elicitation. We choose a pair $(x_i, x_j)$ from the positive dataset $\mc D_1$ in each round. We then ask the subject the following binary question: ``Between two possible recourses $x_i$ and $x_j$, which one do you prefer to implement?''. The answer from the subject takes one of the three answers: $x_i$ or $x_j$ or indifference. The subject's answer can be used to learn a binary preference relation $\mc P$. If $x_i$ is preferred to $x_j$, then we denote $x_i \mc P x_j$; if the subject is indifferent between $x_i$ and $x_j$, then we have simultaneously $x_i \mc P x_j$ and $x_j \mc P x_i$. Because both $x_i$ and $x_j$ have positive predicted outcomes, we assume that the subject's preference is solely based on which recourse requires less effort. Assume that $x_i \mc P x_j$, then $A_0$ should satisfy
\be \label{eq:A_0}
        (x_i - x_0)^\top A_0 (x_i - x_0) \le (x_j - x_0)^\top A_0 (x_j - x_0).
\ee
However, to model possible error in the judgment of the subject and to accommodate the indifference answer, we will equip a positive margin $\eps > 0$, and we have $x_i \mc P x_j$ if and only if:
\be \label{eq:ij}
         (x_i - x_0)^\top A_0 (x_i - x_0) \le (x_j - x_0)^\top A_0 (x_j - x_0) + \eps.
\ee
Let us denote the following matrix $M_{ij} \in \Sym^d$ as 
\be \label{eq:M}
        M_{ij} = x_i x_i^\top - x_j x_j^\top +  (x_j - x_i) x_0^\top +  x_0 (x_j - x_i)^\top,
\ee
then we can rewrite~\eqref{eq:ij} in the form $\inner{A_0}{M_{ij}} \le \eps$. Let $\mbb{P}$ be a set of ordered pairs representing the information collected so far about the preference of the subject:
    \[
    \mbb P = \left\{ (i, j) \in \llbracket N \rrbracket \times \llbracket N \rrbracket~:~x_i \mc P x_j \right\}.
    \]
For any preference set $\mbb P$, we can define $\mc U_{\mbb P}$ as the set of possible cost matrices $A$ that is consistent with the revealed preferences $\mbb P$:
\be \label{eq:U-def}
    \mc U_{\mbb P} \Let \{ A \in \PSD^d ~:~ \inner{A}{M_{ij} } \le \eps ~\forall (i, j) \in \mbb P\},
\ee
then at any time, we have $A_0 \in \mc U_{\mbb P}$. Thus, $\mc U_{\mbb P}$ is considered the confidence set of the cost matrix from the viewpoint of the recourse generation framework. Our learning framework aims to reduce the size of $\mc U_{\mbb P}$, hoping to pinpoint a small region where $A_0$ may reside. Afterward, we use a recourse generation method adapted to the confidence set $\mc U_{\mbb P}$.

We present the overall flow of our framework in Figure~\ref{fig:illustration}. In general, our framework addresses several questions of the cost-adaptive recourse-generation approach:
\begin{enumerate}[leftmargin =5mm]
    \item What are the questions to ask the subject? If $N$ is large, asking the subject exhaustively for $\mc O(N^2)$ pairwise comparisons is impossible. Thus, this question aims to find the pair $x_i$ and $x_j$ such that $(i, j) \notin \mbb P$ and $(j, i) \notin \mbb P$, and that adding either one of these two ordered pairs to $\mbb P$ will bring the largest amount of information as possible (in the sense of narrowing down the set $\mc U_{\mbb P}$).
    \item How to recommend a recourse $x_r$ that minimizes the cost, knowing the confidence set $\mc U_{\mbb P}$?
    \item What happens if there is inconsistency in the subject's preferences? For example, if there exist three distinct indices $(i, j, k)$ such that the subject states $x_i \mc P x_j$, $x_j \mc P x_k$ and $x_k \mc P x_i$. 
\end{enumerate}

The first and third questions are the fundamental questions in preference learning literature~\citep{ref:lu2021review, ref:bertsimas2013learning, ref:vayanos2020robust}. In the marketing literature~\citep{ref:toubia2003fast, ref:toubia2004polyhedral} or recommendation systems literature~\citep{ref:zhao2016user, ref:rashid2008learning, ref:pu2012evaluating}, the preference learning framework aims to recommend products that maximize the utility or preference of subjects. In the adaptive questionnaire framework, we would like to ask questions that give us the most information regardless of the response because the responses to each question are unknown. Moreover, we would like to select the next comparison questions to ask the subject that can maximize the acquired information and reduce the size of the confidence set as quickly as possible~\citep{ref:bertsimas2013learning, ref:vayanos2020robust}.

Guided by these ideas, we integrate the adaptive preference learning framework into the recourse generation problem. We show the overall flow of our framework in Figure~\ref{fig:framework}. Our approach generally consists of two phases: preference elicitation and recourse generation. Next, we present the preference elicitation phase in Section~\ref{sec:cost-identify} and recourse-generation methods in Section~\ref{sec:generation}.

\begin{figure*}[!ht]
        \centering
        \includegraphics[width=0.9\linewidth]{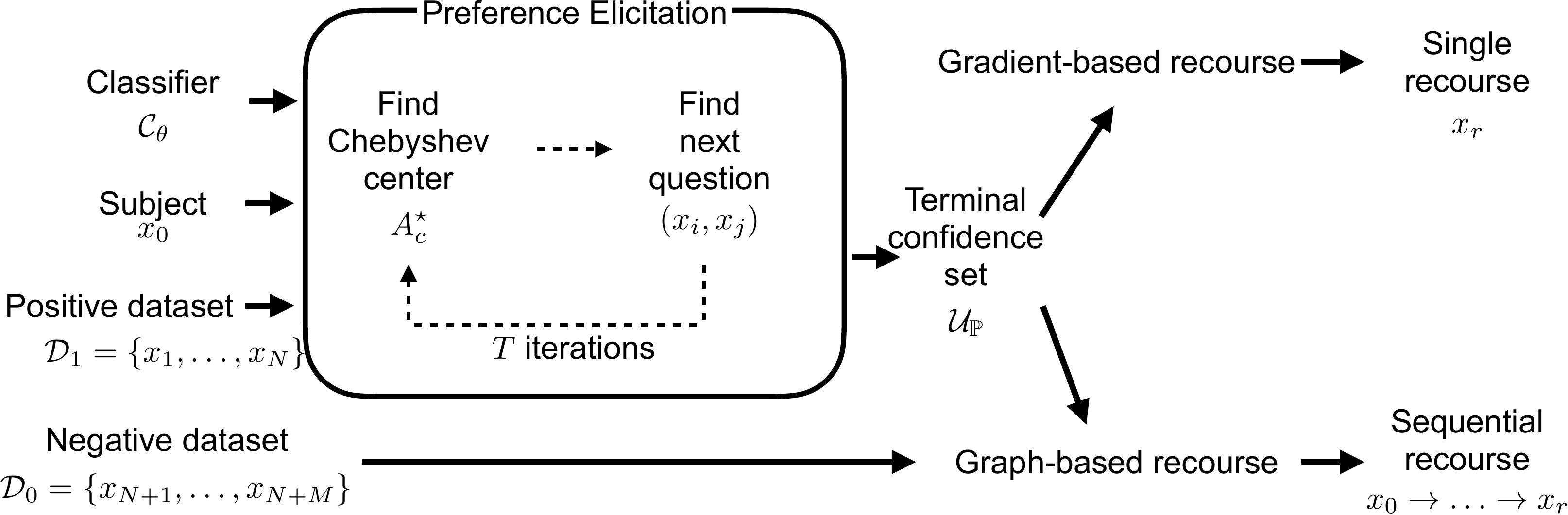}
        \caption{Overall flow of our cost-adaptive recourse recommendation framework. The subject inputs an instance $x_0$. In each of $T$ rounds of question-answer, we first find the Chebyshev center of the set $\mc U_{\mbb P}$, then select the next question that minimizes the distance to the Chebyshev center. We provide two methods to generate the cost-adaptive recourse: gradient-based and graph-based.}
        \label{fig:framework}
    \end{figure*}

\section{Cost Identification via Adaptive Pairwise Comparisons} \label{sec:cost-identify}

\subsection{Finding the Chebyshev Center} \label{sec:chebyshev}


\begin{figure}
    \centering
    \includegraphics[width=0.5\linewidth]{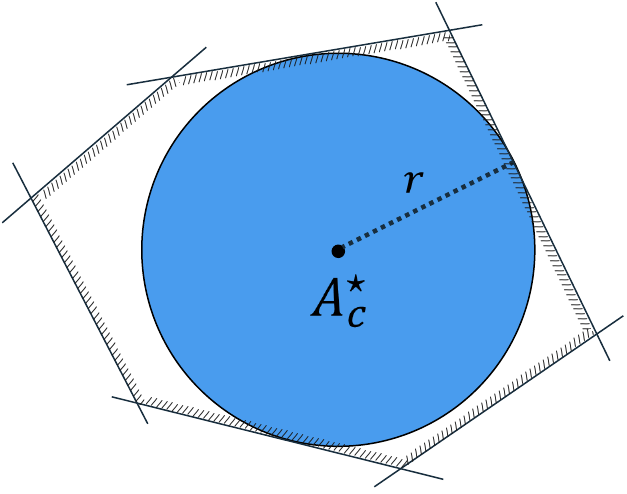}
    \caption{Illustration of the Chebyshev center. Black lines represent the hyperplanes $\inner{A}{M_{ij}} = \eps$ for $(i, j) \in \mbb P$ defining the boundaries of the polytope $\mathcal U_{\mbb P}$. The ball centered at the Chebyshev center $A_c\opt$ with radius $r$ is the largest inscribed ball of $\mathcal U_{\mathbb P}$.}
    \label{fig:illustration}
\end{figure}

First, we observe that without any loss of generality, we can impose an upper bound constraint $A \preceq I$ to the set $\mc U_{\mbb P}$. Indeed, the inequality~\eqref{eq:A_0} is invariant with any positive scaling of the matrix $A_0$, and thus, we can normalize $A_0$ so that it has a maximum eigenvalue of one. Adding $A \preceq I$ makes the set $\mc U_{\mbb P}$ bounded. Given a bounded set $\mc U_{\mbb P}$, we find the Chebyshev center of $\mc U_{\mbb P}$ for each question-answer round. Then, we find the question prescribing a hyperplane closest to this center; thus, this hyperplane can be considered the most aggressive cut. Notice that a question involving $x_i$ and $x_j$ can be represented by the hyperplane $\inner{A}{M_{ij}} = 0$. The confidence set $\mc U_{\mbb P}$ is simply a polytope in the space of positive definite matrices.

We first consider finding the Chebyshev center of the set $\mc U_{\mbb P}$. For any bounded set with a non-empty interior, the Chebyshev center is the center of a ball with the largest radius inside the set. Thus, given a confidence set $\mc U_{\mbb P}$, its Chebyshev center represents a safe point estimate of the true cost matrix. The Chebyshev center $A_c\opt$ and its corresponding radius $r\opt$ of $\mc U_{\mbb P}$ is the optimal solution of the problem


\[
(A_c\opt, r\opt)\!=\!\arg\max\limits_{ A_c \in \PSD^d,~r \in \R_+ }~\left\{ r ~:~ \|A - A_c \|_F^2 \le r^2 ~~ \forall A \in \mc U_{\mbb P}\right\}.
\]

For our problem, the Chebyshev center can be found by solving a semidefinite program resulting from the following theorem. 
\begin{theorem}[Chebyshev center] \label{thm:chebyshev}
Suppose that $\mc U_{\mbb P}$ has a non-empty interior. The Chebyshev center $A_c\opt$ of the set $\mc U_{\mbb P}$ can be found by solving the following semidefinite program 
\be \label{eq:chebyshev-sdp}
    \begin{array}{cll}
        \max & r \\
        \st & A_c \in \PSD^d,~r \in \R_+ \\
        & A_c \preceq I, \quad \inner{A_c}{M_{ij}} + r \| M_{ij} \|_F \le \eps  \quad \forall (i,j) \in \mbb P.
    \end{array}
\ee
\end{theorem}

\begin{proof}{Proof of Theorem~\ref{thm:chebyshev}.}
Using the definition of the set $\mathcal U_{\mbb P}$ as in~\eqref{eq:U-consistent-def}, the optimization problem to find the Chebyshev center and its radius can be rewritten as
\[
    \begin{array}{cl}
    \max & r \\
     \st & A_c \in \PSD^d,~r \in \R_+ \\
     & \inner{A_c + \Delta}{M_{ij}} \le \eps ~\forall \Delta \in \mathcal B_r,~\forall (i,j) \in \mbb P.
    \end{array}
\]
    where $\mathcal B_r$ is a ball of symmetric matrices with Frobenius norm bounded by $r$: 
    \[
    \mathcal B_r = \{\Delta \in \Sym^d: \| \Delta \|_F \le r\}.
    \]
    Pick any $(i,j) \in \mbb P$, the semi-infinite constraint
    \[
    \inner{A_c + \Delta}{M_{ij}} \le \eps ~\forall \Delta \in \mathcal B_r
    \]
    is equivalent to the robust constraint
    \[
    \inner{A_c}{M_{ij}} + \sup\limits_{\|\Delta\|_F \le r} \inner{\Delta}{M_{ij}} \le \eps.
    \]
    Because the Frobenius norm is a self-dual norm, we have
    \[
    \sup\limits_{\|\Delta\|_F \le r} \inner{\Delta}{M_{ij}} = r \| M_{ij} \|_F.
    \]
    Replacing the above equation to the optimization problem completes the proof.
\end{proof}

\subsection{Recourse-Pair Determination} \label{sec:heuristics}

Finding the next question to ask the subject is equivalent to finding two indices $(i, j) \in \llbracket N \rrbracket \times \llbracket N \rrbracket$, corresponding to two recourses $x_i$ and $x_j$ in the positive dataset $\mc D_1$, such that the corresponding hyperplane $\inner{A}{M_{ij}} = 0$ is as close to the Chebyshev center $A_c\opt$ as possible. This is equivalent to solving the minimization problem
\[
    \min\limits_{(i,j) \in \llbracket N \rrbracket \times \llbracket N \rrbracket}~\frac{|\inner{A_c\opt}{M_{ij}}|}{\| M_{ij} \|_F},
\]
where the matrix $M_{ij}$ is calculated as in~\eqref{eq:M}. The objective function of the above problem is simply the projection distance of $A_c\opt$ to $\inner{A}{M_{ij}} = 0$ under the Frobenius norm. 

\textbf{Similar cost heuristics.} An exhaustive search over all pairs of indices $(i,j)$ requires an $\mc O(N^2)$ complexity. This search may become too expensive for large datasets because we must conduct one separate search at each round. We propose a heuristic that can produce reasonable questions in a limited time to alleviate this burden. This heuristics is based on the following observation: given an incumbent Chebyshev center $A_c\opt$, two valid recourses $x_i$ and $x_j$ are more comparable to each other if their resulting costs measured with respect to $A_c\opt$ are close to each other, that is, $c_{A_c\opt}(x_i, x_0) \approx c_{A_c\opt}(x_j, x_0)$. If their costs are too different, for example,  $c_{A_c\opt}(x_i, x_0) \ll c_{A_c\opt}(x_j, x_0)$, then it is highly likely that the subject will prefer $x_i$ to $x_j$ uniformly over the set of possible weighting matrices in $\mc U_{\mbb P}$. Profiting from this observation, we consider the following similar-cost heuristic:
\begin{itemize}[leftmargin = 5mm]
    \item Step 1: Compute the distances from $x_i$ to $x_0$: $s_i = (x_i - x_0)^\top A_c\opt (x_i - x_0)$ for all $i \in \llbracket N \rrbracket$,
    \item Step 2: Sort $s_i$ in a non-decreasing order. The sorted vector is denoted by $(s_{[1]}, \ldots, s_{[N]})$,
    \item Step 3: For each $i = 1, \ldots, N-1$, choose a pair of adjacent cost samples $x_{[i]}$ and $x_{[i+1]}$ corresponding to $s_{[i]}$ and $s_{[i+1]}$, then compute the projection distance of the incumbent center $A_c\opt$ to the hyperplane $\inner{M_{[i],[i+1]}}{A} = 0$.
    \item Step 4: Pick a pair of $([i], [i+1])$ that induces the smallest projection distance in Step 3.
\end{itemize}

In step 2, sorting costs $\mc O(N \log N)$. Nevertheless, in Step 3, we only need to compute $N$ times the projection distance by looking at pairs of adjacent costs, contrary to the total number of $\mc O(N^2)$ pairs. We provide an experiment comparing similar cost heuristics and exhaustive search in Section~\ref{sec:experiment}.

\section{Cost-Adaptive Recourse Recommendation} \label{sec:generation}

Given the subject input $x_0$, this section explores two generalizations to generate single and sequential recourses, adapted to the terminal confidence set $\mc U_{\mbb P}$ of the cost metric. In Section~\ref{sec:gradient}, we generalize the gradient-based counterfactual generation method in~\citet{ref:wachter2018counterfactual}. In Section~\ref{sec:graph}, we generalize the graph-based counterfactual generation method in~\citet{ref:poyiadzi2020face}.

\subsection{Gradient-based Cost-adaptive Single Recourse} \label{sec:gradient}

Given a machine learning model $f_{\theta}: \R^d \to (0, 1)$ that outputs the probability of being classified in the favorable group. The binary classifier $\mc C_{\theta}: \R^d \to \{0, 1\}$ takes the form of a threshold policy
\[
    \mc C_\theta(x) = \begin{cases}
        1 & \text{if } f_\theta(x) \ge 0.5, \\
        0 & \text{otherwise,}
    \end{cases}
\]
where we have used a threshold of $0.5$ similar to the setting in \citet{ref:wachter2018counterfactual}.

We suppose that we have access to the probability output $f_\theta$. Let $l$ be a differentiable loss function that minimizes the gap between $f_\theta(x)$ and the decision threshold 0.5; one can think of $l(f_\theta(x), 1)$ as the term that promotes the validity of the recourse. Given a weight $\lambda \ge 0$ which balances the trade-off between the validity and the (worst-case) cost, we can generate a recourse for an input instance $x_0$ by solving
\be \label{eq:recourse_relax}
        \Min{x \in \mc X}~\left\{l(f_\theta(x), 1) + \lambda \Max{A \in \mc U_{\mbb P}}(x - x_0)^\top A (x - x_0) \right\}.
\ee
In problem~\eqref{eq:recourse_relax}, we can use any loss function $l$ that is differentiable in $x$. A practical choice for loss function is~the quadratic loss $l(f_\theta(x), 1) = (f_\theta(x) - 0.5)^2$, which is a differentiable function in $x$~\citep{ref:wachter2018counterfactual}. Alternatively, we can use the hinge loss to measure the discrepancy between $f_\theta(x)$ and the target probability threshold $0.5$. Under a mild condition about the uniqueness of the optimal solution to the inner maximization problem, the cost term in the objective of~\eqref{eq:recourse_relax} is also differentiable. Thus, one can invoke a (projected) gradient descent algorithm to solve~\eqref{eq:recourse_relax} and find the recourse. Algorithm~\ref{alg:gd} proceeds iteratively to solve problem~\eqref{eq:recourse_relax}. In each iteration, we first find a matrix $A\opt$ of the max problem with a solver such as Mosek~\citep{ref:mosek}, and then we take a gradient step in the variable $x$ using the computed gradient. The next incumbent solution is the projection onto the set $\mathcal X$, where $\Pi_{\mathcal X}$ denotes the projection onto $\mathcal X$. Furthermore, similar to~\citet{ref:wachter2018counterfactual}, we can add an early stopping criterion for Algorithm~\ref{alg:gd}. For example, we can stop the algorithm at iteration $t$ as soon as $\mathcal C_\theta(x^t)=1$.

\begin{algorithm}[H]
	\caption{Gradient descent algorithm for cost-adaptive recourse generation}
	\label{alg:gd}
	\begin{algorithmic}
		\STATE {\bfseries Input:} Input $x_0$ such that $\mathcal C_\theta(x_0)=0$
            \STATE {\bfseries Parameters:}  $\lambda > 0$, learning rate $\alpha$
		\STATE {\bfseries Initialization:} 
		Set $x^{0} \leftarrow x_0$ 
	\FOR{$t =0, \ldots, T-1$}
            \STATE Set $A\opt \leftarrow \max_{A \in \mathcal U_{\mbb P}}~(x^{t} - x_0)^\top A (x^{t} - x_0)$
            \STATE Set $g\leftarrow \nabla l(f_\theta(x^{t}), 1) + 2\lambda A\opt (x^{t} - x_0)$
            \STATE Set $x^{t+1} \leftarrow \Pi_{\mathcal X}(x^{t} - \alpha g)$
        \ENDFOR
		\STATE{\bfseries Output:} $x^T$
	\end{algorithmic}
 
\end{algorithm}

\subsection{Graph-based Cost-adaptive Sequential Recourse} \label{sec:graph}

\begin{figure}
    \centering
    \includegraphics[width=0.5\linewidth]{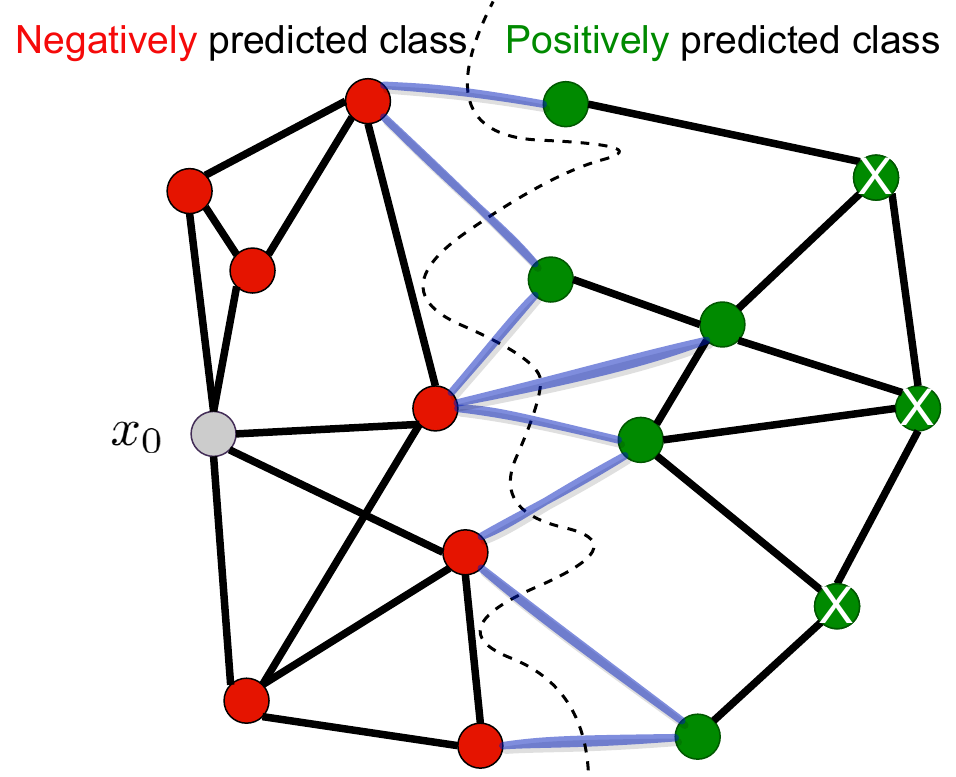}
    \caption{The illustration of $\mathcal{G}$ shows negatively predicted samples as red circles and positively predicted samples as green circles. The input instance $x_0$ is a gray circle. The terminal edges and unreachable nodes of flows in $\mathcal{F}$ are blue edges and green nodes with white crosses, respectively.}
    \label{fig:graph_based_m}
\end{figure}

In Section~\ref{sec:gradient}, we introduce a gradient-based recourse-generation method. However, this approach requires access to the gradient information, which is restricted in some real-world applications~\citep{ref:ilyas2018black, ref:alzantot2019genattack}. In this section, we present a model-agnostic recourse-generation approach that leverages the ideas from FACE~\citep{ref:poyiadzi2020face}. After $T$ rounds of questions in Section~\ref{sec:cost-identify}, we solve problem~\eqref{eq:chebyshev-sdp} to find the Chebyshev center $A\opt$ of the terminal confidence set $\mc U_{\mbb P}$. 

\textbf{Graph construction.} We first build a directed graph $\mathcal G = (\mathcal V, \mathcal E)$ that represents the geometry of the available data: each node $x_i \in \mathcal V =  \{x_0\} \cup \mathcal D_1 \cup \mathcal D_0$ corresponds to a data sample, and an edge $(x_i, x_j) \in \mathcal E$ represents a feasible transition from node $x_i$ to node $x_j$. We compute the edge weight $w_{ij}=c_{A\opt}(x_i, x_j)$ based on Mahalanobis cost function associated with matrix $A\opt$. Finally, $w_{ij}=\infty$ for $(x_i, x_j) \notin \mathcal E$.  

\textbf{Sequential recourse generation.} 
Recall that $\mathcal D_1$ is the set of all vertices with favorable predicted outcomes. A one-step recourse recommendation suggests a single continuous action from $x_0$ to $x_r$~\citep[e.g.,][]{ref:ustun2019actionable, ref:mothilal2020explaining}. A sequential recourse is a directed path from the input instance $x_0$ to a node $x_r \in \mathcal D_1$; each transition in the path is a concrete action that the subject has to implement to move towards $x_r$. A sequential recourse has several advantages compared to the one-step ones: plausibility and sparsity. In real-world applications, sequential steps are more plausible than a one-step continuous change~\citep{ref:ramakrishnan2020synthesizing, ref:singh2021directive}. Moreover, recent work shows that sequential recourse promotes sparsity, allowing subjects to modify a few features at each step~\citep{ref:verma2022amortized}. For an illustration of sequential recourse in a medical application, please refer to Figure~\ref{fig:medical}. The cost of a sequential recourse is computed by the sum of all the edge weights in the path. Thus, we can recommend a sequential and actionable recourse by finding a path that originates from $x_0$ and ends at the node $x_r\opt \in \mathcal D_1$ with the lowest cost. 

\textbf{Worst-case sequential recourse generation}. After conducting $T$ rounds of questions in Section~\ref{sec:cost-identify}, we obtain confidence set $\mathcal{U}_{\mathbb{P}}$ for the parameter $A_0$. However, the precise value of $A_0$ remains unknown. In this section, we focus on minimizing the total cost of the sequential recourse subject to the most unfavorable scenario of $A_0$ within the final confidence set. 

Let $\mathcal F$ denote the set containing all possible flows from the input subject $x_0$ to a node in $\mathcal D_1$.  Mathematically, we can write $\mathcal F$ as
\[
\mathcal F = \left\{ f_{ij} \in \{0, 1\}~ \forall (x_i, x_j) \in \mathcal E: \begin{array}{l}
        \sum_{(x_0, x_j) \in \mathcal E} z_{0j} - \sum_{(x_j, x_0) \in \mathcal E} z_{j0}= 1 \\
        \sum_{x_j \in \mathcal D_1, x_i \in \mathcal V \backslash \mathcal D_1} z_{ij} = 1 \\
        \sum_{(x_i, x_j) \in \mathcal E} z_{ij} - \sum_{(x_j, x_i) \in \mathcal E} z_{ji} = 0 ~~\forall x_i \in \mathcal V \backslash \mathcal D_1, x_i \neq x_0
\end{array}
\right\}.
\]

Figure~\ref{fig:graph_based_m} illustrates the visual representation of the set $\mathcal{F}$. The first constraint ensures that the total flow out of $x_0$ is precisely one. The second constraint enforces the terminal condition for flows, halting the flow once it reaches the first node in the positive class. In the visual depiction in Figure~\ref{fig:graph_based_m}, the terminal edges of flows are visually distinguished as blue edges. Consequently, positive nodes without direct connections from negative nodes are not part of any flows, and they are identifiable as green nodes with white crosses in Figure~\ref{fig:graph_based_m}. The third constraint imposes flow conservation at each negatively predicted node. For any $f \in \mathcal F$, we have $f_{ij} = 1$ if the edge $(x_i, x_j)$ constitutes one (actionable) step in the path. The optimal cost-robust sequential recourse is defined to be the optimal flow of the min-max problem
\be \label{eq:sp}
        \min_{f \in \mathcal F}~\max \limits_{A \in \mathcal U_{\mbb P}} \sum_{(x_i, x_j) \in \mathcal E} w_{ij}(A) f_{ij},
\ee
with the edge weight depends explicitly on the weighting matrix $A$ as
$ w_{ij}(A) = (x_i - x_j)^\top A (x_i - x_j)$. The next proposition asserts an equivalent form of~\eqref{eq:sp} as a single-layer minimization problem.

\begin{proposition}[Equivalent formulation]
    \label{prop:sp}
    Problem~\eqref{eq:sp} is equivalent to
    \be
    \label{eq:sp2}
    \left\{
    \begin{array}{cll}
        \min & \inner{U}{I} + \eps \sum_{(i, j) \in \mbb P} t_{ij} \\ 
        \st & f \in \mathcal F,~t_{ij} \ge 0~\forall (i, j) \in \mbb P,~U \in \PSD^d \\
        & U + \sum_{(i, j) \in \mbb P}~M_{ij} t_{ij} \succeq \sum_{(x_i, x_j) \in \mathcal E} (x_i - x_j) (x_i - x_j) ^\top f_{ij}.
    \end{array}
    \right.
\ee
\end{proposition}
\begin{proof}[Proof of Proposition~\ref{prop:sp}]
Semidefinite programming duality asserts that
    \begin{align*}
\max \limits_{A \in \mathcal U_{\mbb P}}~\sum_{(i, j) \in \mathcal E} w_{ij}(A) f_{ij}
=&\left\{
    \begin{array}{cl}
    \max & \inner{A}{\sum_{(i, j) \in \mathcal E}  (x_{i} - x_{j}) (x_{i} - x_{j}) ^\top f_{ij} } \\
    \st & 
    \inner{A}{M_{ij} } \le \eps ~\forall (i, j) \in \mbb P \\
    & 0 \preceq A \preceq I
    \end{array}
\right. \\
=& \left\{
    \begin{array}{cl}
    \min & \inner{U}{I} + \eps \sum_{(i,j) \in \mbb P} t_{ij} \\
    \st & U + \sum_{(i,j) \in \mbb P}~M_{ij} t_{ij} \succeq \sum_{(x_i, x_j) \in \mathcal E}  (x_{i} - x_{j}) (x_{i} - x_{j}) ^\top f_{ij} \\
    & t_{ij} \ge 0~\forall (i,j) \in \mbb P,~U \in \PSD^d.
    \end{array}
\right. 
\end{align*}
Replacing the minimization above into the objective function leads to the postulated result. 
\end{proof}

Problem~\eqref{eq:sp2} is a binary semidefinite programming problem, which is challenging to solve due to its combinatorial nature. Consequently, finding an optimal sequential recourse can be a daunting task. To address this issue, we propose an alternative approach. Specifically, we associate the weight of each edge $(x_i, x_j)$ with its maximum cost taken over all possible values of $A$ in the set $\mathcal{U}_{\mathbb{P}}$:
\begin{align*}
\bar{w}_{ij} = \max \limits_{A \in \mathcal U_{\mbb P}}~w_{ij}(A) = \left\{
    \begin{array}{cl}
    \max & \inner{A}{(x_{i} - x_{j}) (x_{i} - x_{j}) ^\top  } \\
    \st & 0 \preceq A \preceq I, \quad 
    \inner{A}{M_{i'j'} } \le \eps ~\forall (i', j') \in \mbb P.
    \end{array}
\right.
\end{align*}
Given a graph $\mathcal G$ with the worst-case weight matrix $[\bar{w}_{ij}]$, we find the shortest paths from $x_0$ to each positively-predicted node in $\mathcal D_1$. The recommended sequential recourse is the path that originates from $x_0$ and ends at the node $x_r\opt \in \mathcal D_1$ with the lowest path cost.

It is important to note that the terminal set $\mathcal U_{\mathbb{P}}$ is generally not a singleton as it still contains many possible matrices that conform with the feedback information. To find the recourse, we need to borrow the ideas from robust optimization, which formulates the problem as a min-max optimization problem. Looking at the worst-case situation can eliminate any bad surprises regarding the implementation cost. The min-max problem provides an acceptable recourse for challenging scenarios within the domain where $A$ is a matrix satisfying $A \preceq I, A \in \mathbb{S}_{+}^d$. This approach also proves valuable when users' responses contain significant noise and inconsistencies, resulting in a still large search space for $A_0$ in the final round.

\section{Generalizations} \label{sec:generalization}

In this section, we describe two main generalizations of our framework: Section~\ref{sec:inconsistency} considers possible inconsistencies in the preference elicitation of the subject, and Section~\ref{sec:multi} considers the generalization to a $k$-way questioning.

\subsection{Addressing Inconsistency in Cost Elicitation} \label{sec:inconsistency}

It is well-documented that human responses in behavior elicitation may exhibit inconsistencies. Inconsistencies occur when there exist three distinct indices $(i, j, k)$ such that the user states $x_i \mc P x_j$, $x_j \mc P x_k$ and $x_k \mc P x_i$. In this case, the set $\mc U_{\mbb P}$ becomes empty, and finding a Chebyshev center $A_c\opt$ is impossible. One practical approach to alleviate the effect of the inconsistency is to allow a fraction of the stated preferences to be violated in the definition of the cost-uncertainty set $\mc U_{\mbb P}$. Let $| \mbb P |$ denote the cardinality of the set $\mbb P$.  Suppose we tolerate $\alpha$\% of inconsistency, i.e., there are at most $\alpha | \mbb P|$ preferences in the set $\mbb P$ that can be violated. We define $\mc U_{\mbb P}^\alpha$ as the set of possible cost matrices $A$ with at most $\alpha\%$ inconsistency with the preference set $\mbb P$. This set can be represented using auxiliary binary variables as 
\be \label{eq:U-consistent-def}
\mc U_{\mbb P}^\alpha = 
\left\{
A \in \PSD^d: 
\begin{array}{l}
\exists \gamma_{ij} \in \{0, 1\}~~\forall (i,j) \in \mbb P \\
\sum_{(i,j) \in \mbb P} \gamma_{ij} \le \alpha | \mbb P| \\
\inner{A}{M_{ij}} \le \eps + \gamma_{ij} \mathds{M}
\end{array}
\right\},
\ee
where $\mathds{M}$ is a big-M constant. Intuitively, $\gamma_{ij}$ is an indicator variable: $\gamma_{ij}=1$ implies that the preference $x_i \mc P x_j$ is inconsistent, and thus the corresponding halfspace becomes $\inner{A}{M_{ij}} \le \eps + \mathds{M}$, which is a redundant constraint. The Chebyshev center of the set $\mc U_{\mbb P}^\alpha$ can be found by solving a binary semidefinite program. 

\begin{theorem}[Chebyshev center with inconsistent elicitation] \label{thm:chebyshev2}
Given a tolerance parameter $\alpha \in (0, 1)$. The Chebyshev center $A_c\opt$ of the set $\mc U_{\mbb P}^\alpha$ can be found by solving the binary semidefinite program 
\be \label{eq:chebyshev2}
    \begin{array}{cl}
        \max & r \\
        \st & A_c \in \PSD^d,~r \in \R_+,~\gamma_{ij} \in \{0, 1\}~~~~\forall (i,j) \in \mbb P \\
        & \inner{A_c}{M_{ij}} + r \| M_{ij} \|_F \le \eps + \gamma_{ij} \mathds{M} \quad \forall (i,j) \in \mbb P\\
        & \sum_{(i,j) \in \mbb P} \gamma_{ij} \le \alpha | \mbb P | \\
        & A_c \preceq I,
    \end{array}
\ee
where $\mathds{M}$ is a big-M constant.
\end{theorem}

\begin{proof}[Proof of Theorem~\ref{thm:chebyshev2}] Using the definition of the set $\mathcal U_{\mbb P}$ as in~\eqref{eq:U-consistent-def}, the optimization problem to find the Chebyshev center and its radius can be rewritten as
\[
    \begin{array}{cl}
    \max & r \\
     \st & A_c \in \PSD^d,~r \in \R_+ \\
     & \inner{A_c + \Delta}{M_{ij}} \le \eps + \gamma_{ij} \mathds{M} ~\forall \Delta \in \mathcal B_r,~\forall (i,j) \in \mbb P\\
    & \sum_{(i,j) \in \mbb P} \gamma_{ij} \le \alpha | \mbb P |,
    \end{array}
\]
    where $\mathcal B_r$ is a ball of symmetric matrices with Frobenius norm bounded by $r$: 
    \[
    \mathcal B_r = \{\Delta \in \Sym^d: \| \Delta \|_F \le r\}.
    \]
    Pick any $(i,j) \in \mbb P$, the semi-infinite constraint
    \[
    \inner{A_c + \Delta}{M_{ij}} \le \eps + \gamma_{ij} \mathds{M} ~\forall \Delta \in \mathcal B_r
    \]
    is equivalent to the robust constraint
    \[
    \inner{A_c}{M_{ij}} + \sup\limits_{\|\Delta\|_F \le r} \inner{\Delta}{M_{ij}} \le \eps + \gamma_{ij} \mathds{M}.
    \]
    Because the Frobenius norm is a self-dual norm, we have
    \[
    \sup\limits_{\|\Delta\|_F \le r} \inner{\Delta}{M_{ij}} = r \| M_{ij} \|_F.
    \]
    Replacing the above equation to the optimization problem completes the proof.
\end{proof}
Unfortunately, problem~\eqref{eq:chebyshev2} is a binary SDP, and state-of-the-art solvers such as Mosek and GUROBI do not support this class of problem. Adhoc methods to solve binary SDP can be found in~\citet{ref:ni2018mixed} and the references therein.

\subsection{Multiple-option Questions}
\label{sec:multi}

Previous results rely on the pairwise comparison settings: given two valid recourses, $x_i$ and $x_j$, the subject indicates one preferred option. These settings can be easily generalized to $k$-option comparison: Given $k$ distinct indices $i_1, \ldots, i_k$, the subject is asked ``Which recourse among $x_{i_1}, \ldots, x_{i_k}$ do you prefer the most?.'' The answer from the subject will reveal $k-1$ binary preferences: for example, if the subject prefers $x_{i_1}$ the most, then it is equivalent to a revelation of $k-1$ preferences: $x_{i_1} \mc P x_{i_2}, \ldots, x_{i_1} \mc P x_{i_k}$. Thus, if we use a $k$-option question, we can add $k-1$ relations to the set $\mbb P$, which correspond to $k-1$ hyperplanes to the set $\mc U_{\mbb P}$. The computation of the Chebyshev center $A_c\opt$ in Section~\ref{sec:chebyshev} remains invariant. The only added complication is the increased complexity in searching for the next $k$ recourses to ask the subject: instead of $\mc O(N^2)$ questions, the space of possible questions is now of order $N^k$. Fortunately, we can slightly modify the similar cost heuristics to accommodate the $k$-option questions. More specifically, in Step 3 of the heuristics, we can use the following:
\begin{itemize}[leftmargin=5mm]
	\item Step 3: For each $i = 1, \ldots, N-k+1$, choose a tuple of adjacent cost samples $(x_{[i]}, \ldots, x_{[i+k-1]})$ corresponding to $k$-adjacent costs $(s_{[i]}, \ldots, s_{[i+k-1]})$, then compute the \textit{average} projection distance of the incumbent center $A_c\opt$ to the hyperplanes $\inner{M_{[i+k'],[i+k'+1]}}{A} = 0$ for $k'= 0, \ldots, k-2$.
\end{itemize}
The complexity of this heuristics remains $\mc O(N\log(N))$.

\section{Numerical Experiments} \label{sec:experiment}
We evaluate our method, Cost-Adaptive Recourse Recommendation by Adaptive Preference Elicitation (ReAP), using synthetic data and seven real-world datasets: German, Bank, Student, Adult, COMPAS, GMC, and HELOC. Notably, these datasets are commonly used in recourse literature~\citep{ref:verma2022counterfactual,  ref:upadhyay2021towards, ref:mothilal2020explaining}. The main paper presents the results for Synthesis, German, Bank, Student, Adult, and GMC datasets. The results for other datasets can be found in the appendix. For the gradient-based single recourse method in Section~\ref{sec:gradient}, we compare our method to Wachter~\citep{ref:wachter2018counterfactual} and DiCE~\citep{ref:mothilal2020explaining}. For the graph-based sequential recourse method in Section~\ref{sec:graph}, we compare our method to FACE~\citep{ref:poyiadzi2020face} and PEAR~\citep{ref:toni2022generating}. In Appendix~\ref{sec:app-exp}, we present the detailed implementation and numerical results for additional datasets and provide a benchmarking performance for the proposed heuristics. We provide our source code at \url{https://github.com/duykhuongnguyen/ReAP}.

\subsection{Experimental Setup} We follow the standard setup in recourse-generation problem:

\textbf{Data preprocessing.} Following~\citet{ref:mothilal2020explaining}, we preprocess the data using the min-max standardizer for continuous features and one-hot encoding for categorical features.

\textbf{Classifier.} For each dataset, we perform an 80-20 uniformly split (80\% for training) of the original dataset. Then, we train an MLP classifier $\mc C_{\theta}$ on the training data. We use the test data to benchmark the performance of different recourse-generation methods.

\textbf{Cost matrix generation.} We generate $10$ ground-truth matrices $A_0$ with this procedure: first, we generate a matrix $A \in \R^{d \times d}$ of random, standard Gaussian elements, where $d$ is the dimension of $x_0$. Then we compute $A_0 = A A^\top$ and normalize $A_0$ to have a unit spectral radius by taking $A_0 \leftarrow A_0/\sigma_{\max}(A_0)$, where $\sigma_{\max}$ is the maximum eigenvalue function.

For an input $x_0$ and a ground-truth matrix $A_0$, we choose $T$ questions using the similar-cost heuristics in Section~\ref{sec:heuristics} to find the set $\mc U_{\mbb P}$. After $T$ rounds of question-answers, we solve~\eqref{eq:chebyshev-sdp} using MOSEK to find the Chebyshev center $A\opt$ of the terminal confidence set $\mc U_{\mbb P}$. Then, we generate recourse using the gradient-based method in Section~\ref{sec:gradient} and the graph-based method in Section~\ref{sec:graph}. Note that with $T=0$, we haven't asked any questions. Thus, $A\opt=\half I$ (an uninformative estimate). Hence, all algorithms share the same cost function. Within this context, the proposed worst-case sequential recourse generation in Section~\ref{sec:graph} demonstrates the effectiveness as it manages to provide an acceptable recourse for challenging scenarios within the domain where $A$ is a matrix satisfying $A \preceq I, A \in \mathbb{S}_{+}^d$. This approach also proves valuable when users' responses contain significant noise and inconsistencies, resulting in a still large search space for $A_0$ in the final round.

\subsection{Metrics for Comparison} \label{sec:metric}
We compare different recourse-generation methods using the following metrics:

\textbf{Validity.} A recourse $x_r$ generated by a recourse-generation method is valid if $\mc C_{\theta}(x_r)=1$. We compute validity as the fraction of instances for which the recommended recourse is valid.

\textbf{Cost.} For the gradient-based single recourse method, we calculate the cost of a recourse $x_r$ as the Mahalanobis distance between $x_r$ and $x_0$ evaluated with the ground-truth matrix $A_0$ as $c_{A_0}(x_r, x_0)$.

\textbf{Shortest-path cost.} For the graph-based recourse-generation, we report the cost of a sequential recourse $x_0 \to \ldots \to x_r$ as the path cost from input $x_0$ to $x_r$, evaluated with $A_0$.
    
\textbf{Mean rank.} We borrow the ideas from~\citet{ref:bertsimas2013learning} and consider the mean rank metric for ranking recourses based on subject preference. We first rank all of the recourses in the positive dataset $\mc D_1$ with their preferences according to the ground-truth matrix $A_0$. Thus, the recourse with the smallest cost is ranked 1, and the recourse with the largest cost is ranked $N$ ($N$ is the total number of recourses in the positive dataset). We then find the top $K$ recourses according to the cost metric $c_{A\opt}(x, x_0)$ and compare the selected solutions with the true rank of the recourse. Therefore, smaller values indicate that the matrix $A\opt$, the Chebyshev center of the terminal confidence set, is closer to the ground truth $A_0$. Each recourse $x_i \in \mc D_1$ thus can be assigned with a rank $r_i \in [1,\dots,N]$. We compute the normalized mean rank of top $K$ recourses as $r_{\mathrm{mean}} = (\sum_{i=1}^{K} r_i - r_{\min}) / r_{\max}$ where $r_{\min}=\sum_{i=1}^{K} i =(K + 1)K/2$ and $r_{\max}=\sum_{i=N-K+1}^{N} i = (2N - K + 1) K/2$ are normalizing constants so that $r_{\mathrm{mean}} \in (0, 1)$.

\subsection{Numerical Results}
We conducted several experiments to study the efficiency of our framework in generating cost-adaptive recourses. Firstly, we compare our two cost-adaptive recourse-generation methods, gradient-based and graph-based, with the recourse-generation baselines. We also conduct a statistical test to examine the paired difference of the cost between our method and baselines. Then, we study the impact of the number of questions $T$ on the mean rank and compare our two proposed frameworks, two-option and multiple-option questions. Lastly, we conduct the experiments to demonstrate the effectiveness of the worst-case objective proposed in Section~\ref{sec:graph}. Appendix~\ref{sec:app-exp} provides additional numerical results and discussions.

\subsubsection{Gradient-based Cost-adaptive Recourse} In this experiment, we generate recourse using our gradient-based recourse-generation method. We compute the cost as the Mahalanobis distance described in Section~\ref{sec:metric}. We compare our method with three baselines: Wachter and DiCE. We select a total of $T=5$ questions for our ReAP framework. ReAP and Wachter's method share the learning rate $\alpha$ and the trade-off parameter between validity and cost $\lambda$. We opt for $\alpha=0.01$ and $\lambda=1.0$. We use the default setting for the proximity weight and the diversity weight of DiCE with values $0.5$ and $1.0$, respectively. Table~\ref{tab:gd-exp} demonstrates that DiCE has the highest cost across all datasets, and its validity is not perfect in the German, Bank, and Student datasets. Our method has similar validity to Wachter but at a lower cost in three out of four datasets. It is important to note that if $T=0$, the Chebyshev center is $A\opt=\half I$, and the cost metric $c_{A\opt}(x, x_0)$ becomes the squared Euclidean distance between $x$ and $x_0$, which DiCE and Wachter directly optimize. Thus, these results indicate that our approach effectively adjusts to the subject's cost function and adequately reflects the individual subject's preferences.

\begin{table}[ht]
    \centering
    \caption{Benchmark of Cost and Validity between gradient-based methods on four datasets.}
    \label{tab:gd-exp}
    \begin{filecontents*}{mlp_gd.csv}
dataset,method,cost,valid
Synthetic,DiCE,0.31 $\pm$ 0.27,\textbf{1.00} $\pm$ 0.00
,Wachter,0.12 $\pm$ 0.14,\textbf{1.00} $\pm$ 0.00
,ReAP,\textbf{0.10} $\pm$ 0.15,\textbf{1.00} $\pm$ 0.00
German,DiCE,0.10 $\pm$ 0.37,0.96 $\pm$ 0.19
,Wachter,0.03 $\pm$ 0.02,\textbf{1.00} $\pm$ 0.00
,ReAP,\textbf{0.01} $\pm$ 0.01,\textbf{1.00} $\pm$ 0.00
Bank,DiCE,1.43 $\pm$ 0.61,0.99 $\pm$ 0.10
,Wachter,0.11 $\pm$ 0.10,\textbf{1.00} $\pm$ 0.00
,ReAP,\textbf{0.08} $\pm$ 0.08,\textbf{1.00} $\pm$ 0.00
Student,DiCE,0.07 $\pm$ 0.18,0.64 $\pm$ 0.48
,Wachter,\textbf{0.05} $\pm$ 0.07,\textbf{1.00} $\pm$ 0.00
,ReAP,\textbf{0.05} $\pm$ 0.07,\textbf{1.00} $\pm$ 0.00
\end{filecontents*}
\pgfplotstabletypeset[
            col sep=comma,
            string type,
            every head row/.style={before row=\toprule,after row=\midrule},
            every row no 2/.style={after row=\midrule},
            every row no 5/.style={after row=\midrule},
            every row no 8/.style={after row=\midrule},
            every last row/.style={after row=\bottomrule},
            columns/dataset/.style={column name=Dataset, column type={l}},
            columns/method/.style={column name=Methods, column type={l}},
            columns/cost/.style={column name=Cost $\downarrow$, column type={c}},
            columns/valid/.style={column name=Validity $\uparrow$, column type={c}},
        ]{mlp_gd.csv}
\end{table}

\subsubsection{Graph-based Cost-adaptive Recourse} 
In this experiment, we generate recourse using the graph-based sequential recourse method. We compute the cost of a sequential recourse as the shortest-path cost described in Section~\ref{sec:metric}. We compare our graph-based method with FACE. We follow the implementation of CARLA~\citep{ref:pawelczyk2021carla} to construct a nearest neighbor graph with $K=10$. We choose $T=5$ questions for our ReAP method. Table~\ref{tab:graph-exp} demonstrates that our ReAP framework has the lowest cost across all four datasets. The validity of the two methods is perfect in all four datasets because the two methods both find a path from the input node $x_0$ to the node $x_r \in \mathcal D_1$. As mentioned above, if $T=0$, the cost metric $c_{A\opt}(x, x_0)$ becomes squared of the Euclidean distance between $x$ and $x_0$, and FACE builds the graph using this Euclidean metric. These observations show that our graph-based method accurately captures the subjects' preferences and adapts to their cost function.

\begin{table}[ht]
    \centering
    \caption{Benchmark of Path cost between graph-based ReAP and FACE. All methods attain the validity of $1.00\pm0.00$.}
    \label{tab:graph-exp}
        \begin{filecontents*}{mlp_graph.csv}
dataset,method,cost
Synthetic,FACE,0.73 $\pm$ 0.55
,ReAP,\textbf{0.70} $\pm$ 0.56
German,FACE,0.66 $\pm$ 0.48
,ReAP,\textbf{0.53} $\pm$ 0.49
Bank,FACE,1.20 $\pm$ 0.69
,ReAP,\textbf{0.82} $\pm$ 0.39
Student,FACE,1.10 $\pm$ 0.76
,ReAP,\textbf{1.04} $\pm$ 0.66
\end{filecontents*}
\pgfplotstabletypeset[
            col sep=comma,
            string type,
            every head row/.style={before row=\toprule,after row=\midrule},
            every row no 1/.style={after row=\midrule},
            every row no 3/.style={after row=\midrule},
            every row no 5/.style={after row=\midrule},
            every last row/.style={after row=\bottomrule},
            columns/dataset/.style={column name=Dataset, column type={l}},
            columns/method/.style={column name=Methods, column type={l}},
            columns/cost/.style={column name=Path cost $\downarrow$, column type={c}},
        ]{mlp_graph.csv}
\end{table}

\subsubsection{Comparison with PEAR} We implement the PEAR method proposed by~\citet{ref:toni2022generating} based on our understanding of the method and the details outlined in the original paper. We conduct this experiment using Adult and GMC datasets, consistent with their usage in~\citet{ref:toni2022generating}.

Comparing our method to PEAR~\citep{ref:toni2022generating} is not straightforward due to the difference in the cost modeling. Specifically,~\citet{ref:toni2022generating} utilizes a linear structural causal model, whereas we assume the cost function takes the form of the Mahalanobis distance. In this experiment, we employ a Manhattan ($\ell_1$) distance to measure the cost of the actions. In this way, both our method and the PEAR method are misspecified. This experiment aims to benchmark which method is more robust to the misspecification of the cost functional form. As we assume that the subject $x_0$ has the true cost function of the form $c(x, x_0) = \|x - x_0\|_1$, between two recourses $x_i$ and $x_j$, $x_i$ is preferred to $x_j$ if
\[
\|x_i - x_0\|_1 \le \|x_j - x_0\|_1.
\]

Our approach employs the above response model for constructing the terminal confidence set  $\mc U_{\mbb P}$. In contrast, PEAR utilizes the same model to select the optimal intervention in each iteration~\citep[Algorithm~1]{ref:toni2022generating}. Regarding the objective, our method is designed to learn the matrix $A_0$ within the framework of Mahalanobis distance while PEAR's objective is to learn the optimal weights for the cost function outlined in~\citet[Equation (3)]{ref:toni2022generating}.

We choose $T=10$ questions for both methods to ensure a fair comparison. Additionally, since our approach involves pairwise comparisons between recourses, we set the choice set size for PEAR to $2$, which aligns with our method. Following the settings in~\citet{ref:toni2022generating}, the prior distribution of weights takes the form of a mixture of Gaussians with $6$ components, where the means were randomized, and the covariance matrix was set to identity. When $T = 0$, the weights are initialized using the expected prior value.

After we have learned the cost function using each method, we use the graph-based recourse method wherein we construct the graph using the methodology outlined in the FACE method~\citep{ref:poyiadzi2020face}. For our method, we proceed to reassign the weights of the edges $(x_i, x_j) \in \mathcal{E}$ within the graph, employing the cost function $\bar{w}_{ij}=c_{A\opt}(x_i, x_j)$, where $A\opt$ is the Chebyshev center of the terminal confidence set. For PEAR, we reassign the edge weights using the cost function defined in~\citet[Equation (3)]{ref:toni2022generating}. Thus, the two methods can access the same graph structure but different edge costs. We then solve the graph-based recourse problem in Section~\ref{sec:graph}. Finally, we evaluate the path cost from input $x_0$ to $x_r$, evaluated with Manhattan distance, which is the true cost function in this experiment.

\begin{table}
    \centering
    \caption{Benchmark of Path cost between PEAR and graph-based ReAP. All methods attain the validity of $1.00\pm0.00$. Thus, we do not display Validity in the table.}
    \label{tab:graph-detoni}
    \small
    \begin{filecontents*}{mlp_graph_detoni.csv}
dataset,method,cost
Adult,PEAR,1.78 $\pm$ 0.91
,ReAP,\textbf{1.76} $\pm$ 1.02
GMC,PEAR,0.96 $\pm$ 0.52
,ReAP,\textbf{0.81} $\pm$ 0.39
\end{filecontents*}
\pgfplotstabletypeset[
        col sep=comma,
        string type,
        every head row/.style={before row=\toprule,after row=\midrule},
        every row no 1/.style={after row=\midrule},
        every row no 3/.style={after row=\midrule},
        every row no 5/.style={after row=\midrule},
        every last row/.style={after row=\bottomrule},
        columns/dataset/.style={column name=Dataset, column type={l}},
        columns/method/.style={column name=Methods, column type={l}},
        columns/cost/.style={column name=Path cost $\downarrow$, column type={c}},
    ]{mlp_graph_detoni.csv}
\end{table}

Table~\ref{tab:graph-detoni} reports the mean and standard deviation of path cost over $100$ test samples. These results demonstrate that our method ReAP performs comparable to PEAR in the Adult dataset, while we outperform PEAR in the GMC dataset.

\subsubsection{Statistical Test}
We propose a statistical test to look at the paired difference of the cost: for each subject, we compute the ReAP cost and the competing method’s cost. We propose to test the hypotheses:
\begin{itemize}
    \item Null hypothesis $\mathcal H_0$: ReAP cost equals the competing method’s cost.
    \item Alternative hypothesis $\mathcal H_a$: ReAP cost is \textit{smaller} than the competing method’s cost.
\end{itemize}
We use a one-sided Wilcoxon signed-rank test to test the above hypothesis to compare the paired cost values. The p-value of the test between ReAP and the baselines is reported in Table~\ref{tab:stats_test_1} and Table~\ref{tab:stats_test_2}. Suppose we choose the significant level at $0.05$; Table~\ref {tab:stats_test_1} indicates that we should reject the null hypothesis against DiCE and FACE in all three datasets: German, Bank, and Student. ReAP outperforms Wachter in two out of three datasets, except for the Student dataset. This non-rejection does not imply that we should accept the alternative hypothesis that Wachter generates a lower cost than ReAP in the Student dataset. Additionally, ReAP demonstrates superiority over PEAR in the GMC dataset, as seen in Table~\ref{tab:stats_test_2}.

\begin{table}[h!]
    \centering
    \caption{p-values of the Wilcoxon test between ReAP and competing methods (DiCE, Wachter, and FACE) on German, Bank, and Student datasets.}
    \begin{tabular}{lccc}
        \toprule
         Datasets & German &  Bank & Student \\
         \hline
         ReAP vs.~DiCE & 7e-12 & 3e-18 & 0.046 \\
         ReAP vs.~Wachter & 0.0029 & 0.008 & 0.224 \\ 
         ReAP vs.~FACE & 0.009 & 5e-8 & 0.019 \\ 
         \bottomrule
    \end{tabular}
    \label{tab:stats_test_1}
\end{table}

\begin{table}[h!]
    \centering
    \caption{p-values of the Wilcoxon test between ReAP and PEAR on two datasets, Adult and GMC.}
    \begin{tabular}{lcc}
        \toprule
         Datasets & Adult & GMC \\
         \hline
         ReAP vs.~PEAR & 0.442 & 0.0015 \\
         \bottomrule
    \end{tabular}
    \label{tab:stats_test_2}
\end{table}

\subsubsection{Impact of the Number of Questions $T$ to the Mean Rank} 

Here, we analyze the impact of the number of questions $T$ on the mean rank. We first fix the parameter $\eps=0.01$ and vary the number of questions as an integer $T \in [0, 10]$. For each value of $T$, we choose $T$ questions with the heuristics in Section~\ref{sec:heuristics} and solve problem~\eqref{eq:chebyshev-sdp} to find the center $A\opt$. Then, we evaluate the mean rank with $A\opt$. Figure~\ref{fig:mean_rank} demonstrates that the average mean rank decreases as the number of questions increases. This implies that the Chebyshev center $A\opt$ comes closer to the ground truth $A_0$ with the more questions we ask. As a result, the estimate of the actual cost function is more accurate as the number of questions increases. In the preference learning literature, the ideal number of rounds is usually between $3$ and $10$~\citep{ref:bertsimas2013learning}. In Figure~\ref{fig:mean_rank}, we show that as $T$ varies from $0$ to $10$, the estimation of the true cost function is improved. In Table~\ref{tab:gd-exp} and Table~\ref{tab:graph-exp}, at $T=5$, our method outperforms the baselines regarding the recourse cost.

\begin{figure*}[!ht]
        \centering
        \includegraphics[width=1.0\linewidth]{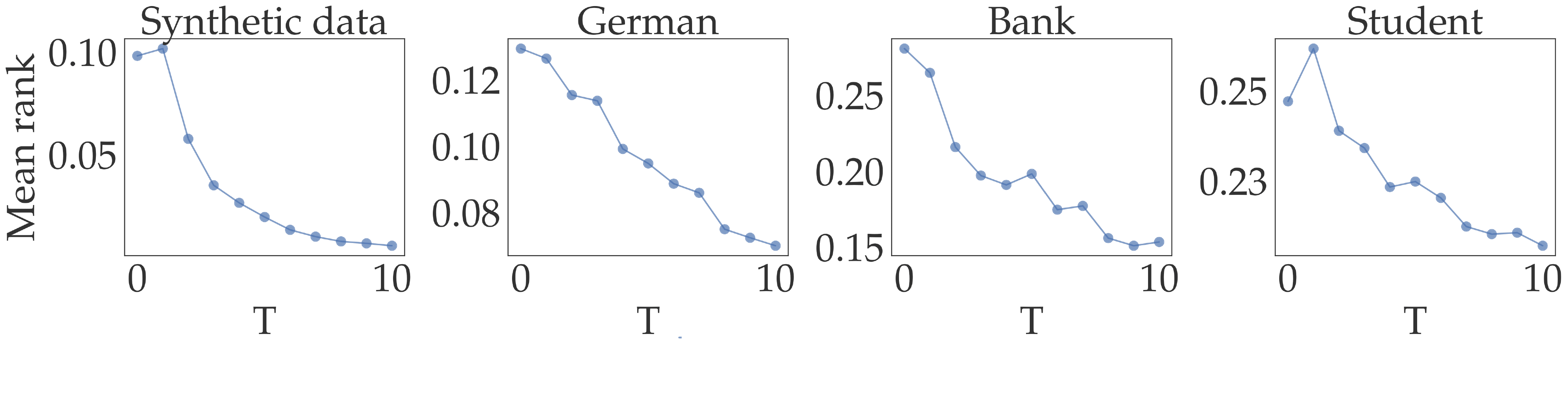}
        \caption{Impact of the number of questions $T$ to the average mean rank on synthetic data and three real-world datasets. As the number of questions increases, the mean rank tends to decrease, highlighting that the Chebyshev center tends closer to the ground truth $A_0$.}
        \label{fig:mean_rank}
\end{figure*}

\subsubsection{Comparison between Two-option and Multiple-option Questions} 
In this experiment, we compare two heuristics for choosing the questions: The recourse-pair heuristics in Section~\ref{sec:heuristics} and multiple-option heuristics in Section~\ref{sec:multi}. We denote the recourse-pair heuristics as ReAP-2 and multiple-option heuristics as ReAP-K. The setting of this experiment is the same as the setting of the experiment, leading to Figure~\ref{fig:mean_rank}.

Figure~\ref{fig:mean_rank_k} demonstrates that as $T$ increases, the mean rank of ReAP-K decreases faster than ReAP-2. Because the complexity of both heuristics is $\mathcal O(N\log(N))$, these results indicate that the multiple-option heuristic is more efficient in our adaptive preference learning framework. This is reasonable because the system gathers more information than a two-option question by asking multiple-option questions. Nevertheless, the subject has a higher cognitive load to answer multiple-option questions.

\begin{figure*}[!ht]
        \centering
        \includegraphics[width=1.0\linewidth]{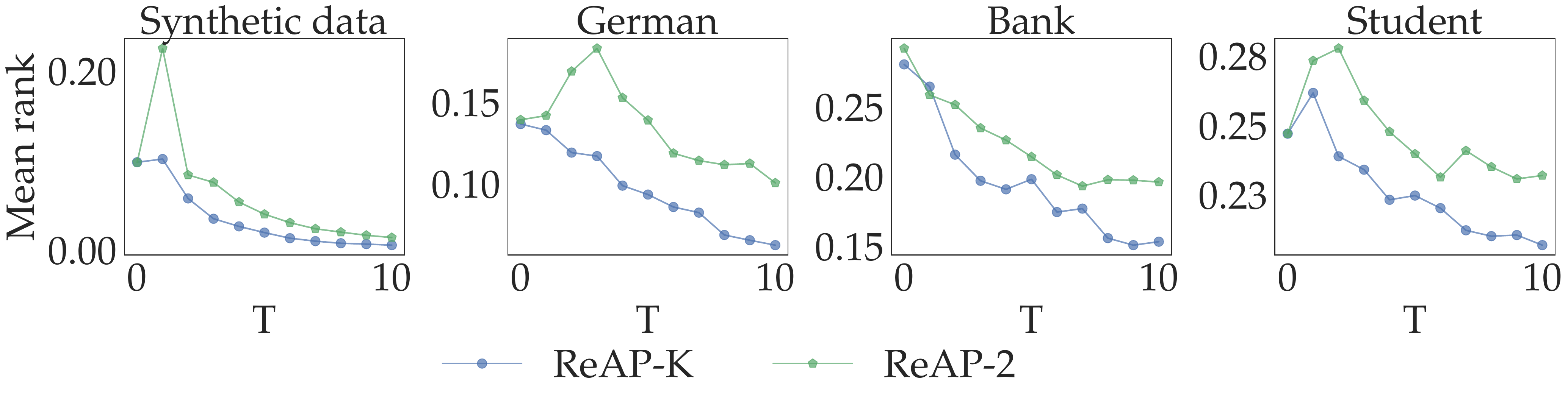}
        \caption{Comparison of two heuristics: recourse-pair heuristics (ReAP-2) and multiple-option heuristics (ReAP-K) with the average mean rank on synthetic data and three real-world datasets.}
        \label{fig:mean_rank_k}
\end{figure*}

\subsubsection{Worst-case Objective}
In this experiment, we study the effectiveness of the proposed worst-case objective in Problem~\eqref{eq:sp}. First, we conduct an experiment to measure the average path lengths for the two methods, employing $A_c\opt$ (ReAP-$A_c\opt$) or solving the worst-case objective (ReAP-WC) in Problem~\eqref{eq:sp}. Table~\ref{tab:wc-path} indicates that the path lengths obtained by solving the worst-case objective are only marginally higher than the alternative, highlighting the effectiveness of our proposed method.

\begin{table}[h!]
    \centering
    \caption{Comparison of average path lengths ($\downarrow$) between two methods on German, Bank, and Student datasets.}
    \begin{tabular}{lccc}
        \toprule
         Datasets & German &  Bank & Student \\
         \hline
         ReAP-$A_c\opt$ & 3.96 & 5.89 & 5.85 \\
         ReAP-WC & 4.14 & 6.07 & 5.91 \\ 
         \bottomrule
    \end{tabular}
    \label{tab:wc-path}
\end{table}

To solve Problem~\eqref{eq:sp} for a small graph, we enumerate all possible flows of set $\mc F$ and solve the inner maximization problem. The optimal solution is the sequential recourse with the lowest worst-case cost. This is a brute-force, exhaustive search method to solve~\eqref{eq:sp}. We can compute the difference between the optimal solutions of solving the Problem~\eqref{eq:sp} and our proposed relaxation using $\bar{w}_{ij}$ as edge cost using the Jaccard distance. The Jaccard distance is popular to measure the dissimilarity between two sets. The results reported in Table~\ref{tab:wc-enum} show that the optimal solutions do not differ significantly between the exhaustive search and the relaxed $\bar{w}_{ij}$ problem.

\begin{table}[h!]
    \centering
    \caption{Jaccard distance between exhaustive search and the relaxed $\bar{w}_{ij}$ problem on three datasets German, Bank and Student.}
    \begin{tabular}{lccc}
        \toprule
         Datasets & German &  Bank & Student \\
         \hline
       Jaccard $\downarrow$ & 0.03 $\pm$ 0.02 &  0.05 $\pm$ 0.06 & 0.05 $\pm$ 0.04 \\
         \bottomrule
    \end{tabular}
    \label{tab:wc-enum}
\end{table}

\section{Conclusions} 
This work proposes an adaptive preference learning framework for the recourse generation problem. Our proposed framework aims to approximate the true cost matrix of the subject in an iterative manner using a few rounds of question-answers. At each round, we select the question corresponding to the most effective cut of the confidence set of possible cost matrices. We provide two recourse-generation methods: gradient-based and graph-based cost-adaptive recourse. Finally, we generalize our framework to handle inconsistencies in subject responses and extend the heuristics to choose the questions from pairwise comparison to multiple-option questions. Extensive numerical experiments show that our framework can adapt to the subject's cost function and deliver strong performance results against the baselines.

\noindent\textbf{Acknowledgments.} Viet Anh Nguyen gratefully acknowledges the generous support from the CUHK’s Improvement on Competitiveness in Hiring New Faculties Funding Scheme and the CUHK's Direct Grant Project Number 4055191.

\appendix

\section{Broader Impacts and Limitations} 
This paper aims to generate recourse adapted to each subject's cost function. The gradient-based method in Section~\ref{sec:gradient} and graph-based method in Section~\ref{sec:graph} require access to gradient information and training data, respectively. We want to highlight that access to this information is leveraged in existing gradient-based methods such as ROAR~\citep{ref:upadhyay2021towards} or graph-based methods such as FACE~\citep{ref:poyiadzi2020face}. A frequent criticism of the needed access to data or model information is that it could violate the privacy of the machine learning system. Moreover, recent research demonstrates that solutions produced by recourse-generation methods and those produced by adversarial example-generating algorithms are highly comparable~\citep{ref:pawelczyk2022exploring}. To increase the system's trustworthiness, a decision-making system must, therefore, be able to discern between an adversarial example and a recourse. We may use various strategies and approaches to ensure privacy to overcome these problems. However, these issues are outside our scope, so we left these problems for future research.

\section{Additional Experiments} \label{sec:app-exp} In this section, we provide the detailed implementation and additional numerical results.

\subsection{Datasets} We present the details about the real-world and synthetic datasets.

\textbf{Real-world data.} We use seven real-world datasets which are popular in the settings of recourse-generation~\citep{ref:mothilal2020explaining, ref:upadhyay2021towards}: German credit~\citep{ref:dua2017uci}, Bank~\citep{ref:dua2017uci}, Student performance~\citep{ref:cortez2008student}, Adult~\citep{ref:becker1996adult}, COMPAS, GMC and HELOC~\citep{ref:pawelczyk2021carla}. We describe the selected subset of features from German, Bank, and Student datasets in Table~\ref{tab:features}. Additionally, we follow the same features selection procedure for Adult, COMPAS Recidivism Racial Bias, Give Me Some Credit (GMC), and HELOC datasets as in~\citet{ref:pawelczyk2021carla}. 

\begin{table}[htb]
\caption{Features selection from German, Bank, and Student datasets in our experiments.}
\label{tab:features}
\vskip 0.15in
\begin{center}
\begin{small}
\begin{tabular}{lcccr}
\toprule
Dataset & Features \\
\midrule
German & Status, Duration, Credit amount, Personal Status, Age \\
Bank & Age, Education, Balance, Housing, Loan, Campaign, Previous, Outcome\\
Student & Age, Study time, Famsup, Higher, Internet, Health, Absences, G1, G2 \\
\bottomrule
\end{tabular}
\end{small}
\end{center}
\vskip -0.1in
\end{table}

\textbf{Synthetic data.} Following previous work~\citep{ref:nguyen2022robust}, we generate the synthetic dataset with two-dimensional data samples by sampling uniformly in a rectangle $x = (x_1, x_2) \in [-2, 4] \times [-2, 7]$ with the following labeling function $f$:
\[
    f(x) = \left\{
            \begin{array}{cl}
                1 & \text{if }  x_2 \ge 1 + x_1 + 2 x_1^2 + x_1^3 - x_1^4, \\
                0 & \text{otherwise}.
            \end{array}
        \right.
\]

\subsection{Implementation Details} Now, we present the implementation details for our methods and baselines in the main paper.

\textbf{Classifier.} We train a three-layer MLP with 20, 50, and 20 nodes and a ReLU activation function in each layer for each dataset. We report the accuracy and AUC of the underlying classifier for each dataset in Table~\ref{tab:acc_clf}.


\begin{table*}[htb]
    \centering
    \caption{Accuracy and AUC of the MLP classifiers on eight datasets.}
    \label{tab:acc_clf}
    \small
    \begin{filecontents*}{accuracy_1.csv}
Dataset,Synthesis,German,Bank,Student,Adult,COMPAS,GMC,HELOC
Accuracy ,0.98,0.72,0.89,0.93,0.85,0.83,0.94,0.74
AUC,0.99,0.62,0.68,0.97,0.9,0.82,0.84,0.81
\end{filecontents*}
\pgfplotstabletypeset[
        col sep=comma,
        string type,
        every head row/.style={before row=\toprule,after row=\midrule},
        every row no 3/.style={after row=\midrule},
        every row no 7/.style={after row=\midrule},
        every last row/.style={after row=\bottomrule},
        columns/data/.style={column name=Dataset, column type={l}},
        columns/data/.style={column name=Synthesis, column type={l}},
        columns/data/.style={column name=German, column type={l}},
        columns/data/.style={column name=Bank, column type={l}},
        columns/data/.style={column name=Student, column type={l}},
        columns/data/.style={column name=Adult, column type={l}},
        columns/data/.style={column name=COMPAS, column type={l}},
        columns/data/.style={column name=GMC, column type={l}},
        columns/data/.style={column name=HELOC, column type={c}},
    ]{accuracy_1.csv}
\end{table*}

\textbf{Reproducibility.} For Wachter, we follow the implementation of CARLA~\citep{ref:pawelczyk2021carla}. Specifically, we initialize $\lambda=1.0$, and then employ an adaptive scheme for $\lambda$ if no valid recourse is found. In particular, if no recourse is found, we reduce the value of $\lambda$ by 0.05, similar to CARLA. Similarly, we follow the implementation of CARLA to construct a nearest neighbor graph with $K=10$ for FACE. We use the original implementation\footnote{\url{https://github.com/interpretml/DiCE}} for DiCE.

Regarding the parameter $\eps$ in our method, we experiment to study the effect of $\eps$ on the cost of final recourse. Table~\ref{tab:eps} shows that the path cost shows only minor variations across different values of $\eps$. 

\begin{table}[h!]
    \centering
    \caption{Effect of $\eps$ on the cost final recourse in German dataset.}
    \begin{tabular}{lc}
        \toprule
         Methods & Path cost \\
         \hline
       ReAP ($\eps=0.01$) & 0.53 $\pm$ 0.49 \\
        ReAP ($\eps=0.02$) & 0.55 $\pm$ 0.51 \\
       ReAP ($\eps=0.05$) &  0.51 $\pm$ 0.45 \\
        ReAP ($\eps=0.1$)) &  0.51 $\pm$ 0.45 \\
         \bottomrule
    \end{tabular}
    \label{tab:eps}
\end{table}

\subsection{Additional Numerical Results}

\subsubsection{Benchmark of Proposed Heuristics}

\textbf{Exhaustive search and similar-cost heuristics.} We compare the run time of the similar-cost heuristics and the exhaustive search for a recourse-pair question. This experiment is conducted on a machine with an i7-10510U CPU.

First, we generate $N$ 2-dimensional data samples for each value $N = 100, \ldots, 10000$. Then, for each value of $N$, we compute the average run time of two methods and report the results in Figure~\ref{fig:time}. We can observe that at $N=10000$, the exhaustive search requires more than $40$s to search for a question, which is impractical in real-world settings.

\begin{figure}[!ht]
        \centering
        \includegraphics[width=0.5\linewidth]{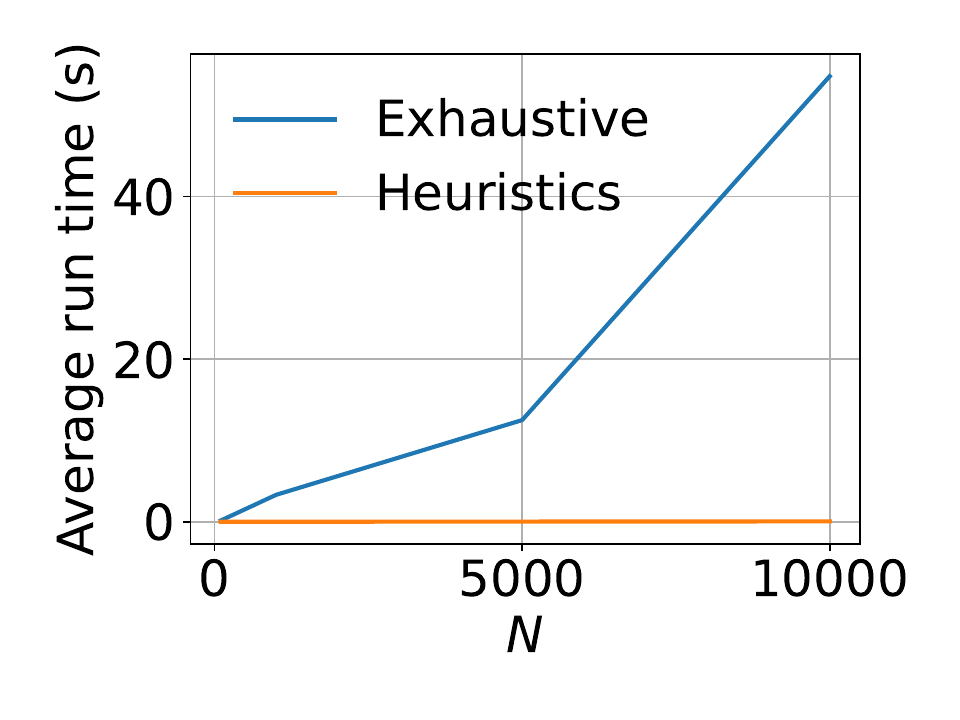}
        \caption{Run time comparison of two heuristics: recourse-pair heuristics (ReAP-2) and multiple-option heuristics (ReAP-K) with the average mean rank on four datasets.}
        \label{fig:time}
\end{figure}

\begin{table}[H]
\caption{The suboptimality gap between the objective of exhaustive search and similar-cost heuristics with inconsistency threshold $\gamma=0.01$ in four datasets.}
\label{tab:heuristics-cons}
\vskip 0.15in
\begin{center}
\begin{small}
\begin{tabular}{lcccr}
\toprule
& Synthetic & German & Bank & Student \\
\midrule
Gap & 0.001 & 0.017 & 0.021 & 0.029 \\
\bottomrule
\end{tabular}
\end{small}
\end{center}
\vskip -0.1in
\end{table}

Let $\mathrm{obj}_{e}$ and $\mathrm{obj}_{h}$ be the optimal values for exhaustive search and similar-cost heuristics, respectively. We compare the relative suboptimality gap between the objective of those two methods as the following:
\[
\mathrm{gap}(\mathrm{obj}_{e}, \mathrm{obj}_{h})=\frac{|\mathrm{obj}_{e} - \mathrm{obj}_{h}|}{\mathrm{obj}_{e}}.
\]
The experiment results show that the suboptimality gap between the objective of the two methods is approximately of order $10^{-6}$ in all datasets. These results demonstrate empirically that our heuristic method can generate good solutions to the original problem at a fraction of the computational time.


\begin{table}[h!]
    \centering
    \caption{Comparison of mean rank ($\downarrow$) between the random selection strategy and our proposed heuristics.}
    \begin{tabular}{lcccc}
        \toprule
         Datasets & Methods &   $T=0$ & $T=5$ & $T=10$ \\
         \hline
       German & Random & 0.26 & 0.24 & 0.15 \\
        & Ours & 0.13 & 0.10 &  0.07 \\
        \hline
        Bank & Random & 0.42 &  0.31 & 0.19 \\
        & Ours &  0.27 & 0.20 &  0.15 \\
        \hline
        Student & Random & 0.37 & 0.31 &  0.29 \\
        & Ours & 0.25 & 0.23 &  0.21 \\
         \bottomrule
    \end{tabular}
    \label{tab:heuristics-rand}
\end{table}

\textbf{Heuristics to address human inconsistencies.} To account for similarity-dependent uncertainty, we can adapt our heuristics by taking into consideration only an adjacent pair of $([i], [i+1])$ for $i \in \llbracket N \rrbracket$ if the disparity between their costs is larger than an inconsistency threshold, denoted as $\gamma$.

We compare the objective values, in terms of their difference, of those two methods in Table~\ref{tab:heuristics-cons}. These results demonstrate that the proposed heuristic method can generate reasonable solutions to the original problem at a fraction of the computational time compared to the exhaustive search.

\textbf{Comparison with random selection strategy.} We show the comparison of mean rank between the random strategy to select the recourse pairs and the proposed heuristics for $T=0$, $T=5$, and $T=10$ in Table~\ref{tab:heuristics-rand}. These results show that as $T$ increases, the mean rank of our approach decreases faster than the random strategy, indicating that the proposed heuristic is more efficient in our adaptive preference learning framework.

\begin{table}[ht]
    \centering
    \caption{Benchmark of Cost and Validity between gradient-based methods on four datasets.}
    \label{tab:gd-exp-supp}
    \begin{filecontents*}{mlp_gd_supp.csv}
dataset,method,cost,valid
Adult,DiCE,2.89 $\pm$ 1.42,\textbf{1.00} $\pm$ 0.00
,Wachter,0.06 $\pm$ 0.04,\textbf{1.00} $\pm$ 0.00
,ReAP,\textbf{0.04} $\pm$ 0.05,\textbf{1.00} $\pm$ 0.00
COMPAS,DiCE,0.51 $\pm$ 1.32,\textbf{1.00} $\pm$ 0.00
,Wachter,\textbf{0.03} $\pm$ 0.04,\textbf{1.00} $\pm$ 0.00
,ReAP,\textbf{0.03} $\pm$ 0.04,\textbf{1.00} $\pm$ 0.00
GMC,DiCE,0.25 $\pm$ 0.16,\textbf{1.00} $\pm$ 0.00
,Wachter,0.02 $\pm$ 0.01,\textbf{1.00} $\pm$ 0.00
,ReAP,\textbf{0.01} $\pm$ 0.00,\textbf{1.00} $\pm$ 0.00
HELOC,DiCE,0.43 $\pm$ 0.22,\textbf{1.00} $\pm$ 0.00
,Wachter,\textbf{0.05} $\pm$ 0.07,\textbf{1.00} $\pm$ 0.00
,ReAP,\textbf{0.05} $\pm$ 0.06,\textbf{1.00} $\pm$ 0.00
\end{filecontents*}
\pgfplotstabletypeset[
        col sep=comma,
        string type,
        every head row/.style={before row=\toprule,after row=\midrule},
        every row no 2/.style={after row=\midrule},
        every row no 5/.style={after row=\midrule},
        every row no 8/.style={after row=\midrule},
        every last row/.style={after row=\bottomrule},
        columns/dataset/.style={column name=Dataset, column type={l}},
        columns/method/.style={column name=Methods, column type={l}},
        columns/cost/.style={column name=Cost $\downarrow$, column type={c}},
        columns/valid/.style={column name=Validity $\uparrow$, column type={c}},
    ]{mlp_gd_supp.csv}
\end{table}

\begin{table}[ht]
    \centering
    \caption{Benchmark of Path cost between graph-based ReAP and FACE. All methods attain the validity of $1.00\pm0.00$. Thus, we do not display Validity in the table.}
    \label{tab:graph-exp-supp}
    \begin{filecontents*}{mlp_graph_supp.csv}
dataset,method,cost
Adult,FACE,0.77 $\pm$ 0.56
,ReUP,\textbf{0.75} $\pm$ 0.52
COMPAS,FACE,0.93 $\pm$ 0.75
,ReUP,\textbf{0.79} $\pm$ 0.61
GMC,FACE,\textbf{0.61} $\pm$ 0.49
,ReUP,0.65 $\pm$ 0.42
HELOC,FACE,1.05 $\pm$ 0.76
,ReUP,\textbf{0.95} $\pm$ 0.65
\end{filecontents*}
\pgfplotstabletypeset[
        col sep=comma,
        string type,
        every head row/.style={before row=\toprule,after row=\midrule},
        every row no 1/.style={after row=\midrule},
        every row no 3/.style={after row=\midrule},
        every row no 5/.style={after row=\midrule},
        every last row/.style={after row=\bottomrule},
        columns/dataset/.style={column name=Dataset, column type={l}},
        columns/method/.style={column name=Methods, column type={l}},
        columns/cost/.style={column name=Path cost $\downarrow$, column type={c}},
    ]{mlp_graph_supp.csv}
\end{table}

\subsubsection{Results on More Datasets}
 Table~\ref{tab:gd-exp-supp} and Table~\ref{tab:graph-exp-supp} report the additional numerical results for four datasets available in CARLA~\citep{ref:pawelczyk2021carla}, including Adult, COMPAS, GMC, and HELOC. These results demonstrate that our method outperforms other baselines, effectively adjusts to the subject’s cost function, and adequately reflects the individual subject’s preferences.

\section{Motivation for the Mahalanobis Cost Function} \label{sec:motivation-M}

We provide two arguments to support the choice of the Mahalanobis cost function. The first argument involves a control theory viewpoint, while the second argument is the connection with the structural causal model.

\subsection{Linear Quadratic Regulator Cost}
\label{sec:LQR}
In this section, we describe a sequential control process that affects feature transitions of a subject $x_0$ towards a target feature while minimizing the cost of efforts. Let $x_0$ and $x_{r}$ be the initial feature of the subject and the target feature. We consider a discrete-time system that, at each iteration, an input effort $u^{(t)}$ drives $x^{(t)}$ to $x^{(t+1)}$
\[
    x^{(t+1)} = x^{(t)} + u^{(t)}, \quad x^{(0)} = x_0.
\]
The objective is to finding the best input efforts $u^{(t)}$ $(\forall t = 0, \ldots, \infty)$ to move from $x_0$ toward $x_r$.   
One can formulate this as solving a Linear Quadratic Regulator (LQR) problem of the form:
\[
    c(x_0, x_r) = \left\{
    \begin{array}{cl}
        \min & \sum_{t=0}^\infty (x^{(t)} - x_r)^\top Q (x^{(t)} - x_r) + (u^{(t)}) ^\top R u^{(t)} \\
        \st & u^{(t)} \in \R^d \quad \forall t = 0, \ldots, \infty\\
        & x^{(t+1)} = x^{(t)} + u^{(t)} \quad \forall t \\
        & x^{(0)} = x_0,
    \end{array}
    \right.
\]
where the parameters $Q$ and $R$ are the subject's state cost and input cost matrices, respectively. The matrix $Q$ is positive semidefinite symmetric while $R$ is positive definite symmetric. The value $c(x_0, x_r)$ is the cost to implement the recourse $x_r$.
\begin{proposition}[Quadratic cost] \label{prop:LQR}
    The optimal cost function $c(x_0, x_r)$ is quadratic, that is:
    \begin{equation*}
        c(x_0, x_r) = (x_0 - x_r)^\top A_0 (x_0 - x_r),
    \end{equation*}
    where $A_0$ is a positive definite symmetric matrix satisfying the following equation:
    \begin{equation*}
            Q = A_0 (R + A_0)^{-1} A_0.
    \end{equation*}
\end{proposition}

Proposition~\ref{prop:LQR} asserts that the minimal cost function has the Mahalanobis form, which solely relies on the initial input $x_0$ and the target features $x_r$. 
\begin{proof}[Proof of Proposition~\ref{prop:LQR}]
    Because $x_r$ is a fixed vector, use the following change of variables $z^{(t)} \leftarrow x^{(t)} - x_r $, we have the equivalence
    \[
        c(x_0, x_r) = \left\{
        \begin{array}{cl}
            \min & \sum_{t=0}^\infty (z^{(t)})^\top Q z^{(t)} + (u^{(t)}) ^\top R u^{(t)} \\
            \st & u^{(t)} \in \R^d \quad \forall t = 0, \ldots, \infty\\
            & z^{(t+1)} = z^{(t)} + u^{(t)} \quad \forall t \\
            & z^{(0)} = x_0 - x_r.
        \end{array}
        \right.
    \]
Let $V(z)$ be the minimum LQR cost-to-go, starting from state $z$ as follows:
\[
    V(z) = \left\{
        \begin{array}{cl}
            \min & \sum_{t=0}^\infty (z^{(t)})^\top Q z^{(t)} + (u^{(t)})^\top R u^{(t)} \\
            \st & u^{(t)} \in \R^d \quad \forall t = 1, \ldots, \infty\\
            & z^{(t+1)} = z^{(t)} + u^{(t)} \quad \forall t \\
            & z^{(0)} = z.
        \end{array}
        \right.
\]
According to~\citet[Section 4.1]{bertsekas2012dynamic}, the function $V$ has a quadratic form $V(z) = z^\top A_0 z$, for some symmetric matrix $A_0$. Because $Q$ is a positive semidefinite symmetric matrix and $R$ is a positive definite symmetric matrix, $V(z) > 0$ for all $z \in \R^d$, meaning that $A_0$ is a positive definite symmetric matrix. Substituting $\sum_{t=1}^\infty (z^{(t)})^\top Q z^{(t)} + (u^{(t)})^\top R u^{(t)}$ by $V(z + u^{(0)})$, we have:
\[
V(z) = \min_{u^{(0)}}~z^\top Q z + (u^{(0)})^\top R u^{(0)} + V(z + u^{(0)}),
\]
which implies that 
\[
z^\top A_0 z = \min_{u^{(0)}} z^\top Q z + (u^{(0)})^\top R u^{(0)} + (z + u^{(0)})^\top A_0 (z + u^{(0)}).
\]
It is easy to see that the objective function of the right-hand side optimization problem is convex. Therefore, the optimal solution of $u^{(0)}$ satisfies
\[
    R u^{(0)} + A_0 (z + u^{(0)}) = 0
 \implies  u^{(0)*} = -(R + A_0)^{-1} A_0 z .
\]
Here, $(R + A_0)$ is invertible because $R$ and $A_0$ are positive definite matrices. Then we have:
\begin{align*}
&z^\top A_0 z = z^\top Q z + (u^{(0)*})^\top R u^{(0)*} + (z + u^{(0)*})^\top A_0 (z + u^{(0)*}) \\
&\Leftrightarrow z^\top A_0 z =  z^\top (Q + A_0 - A_0 (R + A_0)^{-1} A_0) z.
\end{align*}
Therefore, the matrix $A_0$ needs to satisfy the following condition: 
\[
A_0 = Q + A_0 - A_0 (R + A_0)^{-1} A_0
\Leftrightarrow Q = A_0 (R + A_0)^{-1} A_0.
\]
This completes our proof.
\end{proof}

\begin{remark}[Finite time horizon]
    The argument in this section relies on an \textit{in}finite horizon control problem to simplify the discussion. One can formulate a similar finite horizon problem, which leads to a similar Mahalanobis form. The proof follows from an induction argument, which is standard in the control theory literature; see~\citet{bertsekas2012dynamic}.
\end{remark}

\subsection{Casual Graph Recovery}
\label{sec:causal}
This section discusses the connection between the linear Gaussian structural causal model and the Mahalanobis cost function. We consider a linear Gaussian structural equation model (SEM) for the deviation $\delta \in \R^d$ from the initial input $x_0 \in \R^d$ as follows:
\begin{align}
\delta \sim SEM(W_0, D_0)
\Leftrightarrow \delta = W_0 \delta + \epsilon,
\label{eq:delta}
\end{align}
where $\epsilon \sim \mathcal{N}(0, D_0)$ is a multivariate Gaussian with mean vector zero and a covariance matrix $D_0 \in \PD^d$. The $W_0 \in \R^{d \times d}$ is equivalent to the weight $w$ of the structural causal model (SCM) for cost formulation from a directed acyclic graph (DAG) $G$. Each node of $G$ is associated with a single feature, and a nonzero entry $(W_0)_{i, j}$ corresponds to a causal relationship from node $j$ to node $i$. The SEM~\eqref{eq:delta} implies that:
\begin{align*}
    \delta \sim \mathcal{N}(0, (I - W_0)^{-1} D_0 (I - W_0)^{-\top}),
\end{align*}
where $I$ is the identity matrix. The density function for $\delta$ is:
\[
f_0(\delta) = \frac{1}{(2\pi)^{d/2} |(I - W_0)^{-1} D_0 (I - W_0)^{-\top}|^{1/2}} \exp\left(-\frac{1}{2} \delta^{\top} (I-W_0)^{\top} D_0^{-1} (I - W_0) \delta\right). 
\]
Between two deviations $\delta_i = x_i - x_0$ and $\delta_j = x_j - x_0$, the subject prefers a deviation with a higher likelihood, and thus $\delta_i$ is preferred to $\delta_j$ if
\[
\delta_i^{\top} (I-W_0)^{\top} D_0^{-1} (I - W_0) \delta_i \le \delta_j^{\top} (I-W_0)^{\top} D_0^{-1} (I - W_0) \delta_j.
\]
We recover the Mahalanobis cost preference model with $A_0$ corresponding to the precision matrix of the deviation under the linear Gaussian structural equation model. Specifically, the value of $A_0$ is computed as
\[
A_0 = (I-W_0)^{\top} D_0^{-1} (I - W_0).
\]

\end{document}